\documentclass[11pt]{article}
\usepackage{amssymb}

\usepackage{fullpage}
\usepackage{authblk}
\usepackage{MnSymbol}
\usepackage{complexity}
\usepackage{nicefrac}

\usepackage{enumitem}
\usepackage{microtype}
\usepackage{graphicx}
\usepackage{subfigure}
\usepackage{booktabs} 
\usepackage{pifont}

\usepackage{bbm}

\usepackage{natbib}

\usepackage{url}
\usepackage{amsmath}
\usepackage{amssymb}
\usepackage{amsthm}
\usepackage{bbm}
\usepackage{algorithm}
\usepackage[noend]{algorithmic}
\usepackage[algo2e,vlined,ruled,linesnumbered]{algorithm2e}
\usepackage{footmisc}
\usepackage[table]{xcolor}
\makeatletter
\newcommand\footnoteref[1]{\protected@xdef\@thefnmark{\ref{#1}}\@footnotemark}
\makeatother
\usepackage{arydshln}

\SetAlFnt{\small}
\SetAlCapFnt{\small}
\SetAlCapNameFnt{\small}
\SetAlCapHSkip{0pt}
\IncMargin{-\parindent}

\usepackage[hidelinks, colorlinks = true, linkcolor=blue, citecolor=green]{hyperref}
\usepackage{cleveref}
\newcommand{\declarecolor}[2]{\definecolor{#1}{RGB}{#2}\expandafter\newcommand\csname #1\endcsname[1]{\textcolor{#1}{##1}}}
\declarecolor{White}{255, 255, 255}
\declarecolor{Black}{0, 0, 0}
\declarecolor{Maroon}{128, 0, 0}
\declarecolor{Coral}{255, 127, 80}
\declarecolor{Red}{182, 21, 21}
\declarecolor{LimeGreen}{50, 205, 50}
\declarecolor{DarkGreen}{0, 80, 0}
\declarecolor{Purple}{146, 42, 158}
\declarecolor{Navy}{0, 0, 128}
\declarecolor{LightBlue}{84, 101, 202}
\definecolor{mydarkblue}{rgb}{0,0.08,0.45}
\usepackage{tcolorbox}

\usepackage{xcolor}	

\definecolor{CardinalRed}{HTML}{C41E3A}
\definecolor{Dartmouth}{HTML}{00693E}
\definecolor{SapphireBlue}{HTML}{0F52BA}

\colorlet{MyRed}{CardinalRed}
\colorlet{MyGreen}{Dartmouth}

\colorlet{MyLightRed}{MyRed!25}
\colorlet{MyLightGreen}{MyGreen!25}

\colorlet{AlertColor}{MyRed}	
\colorlet{BadColor}{MyRed}	
\colorlet{FocusColor}{MyRed}	
\colorlet{GoodColor}{MyGreen}	
\colorlet{MacroColor}{MyRed}	
\newcommand{\ComBand}{\textsc{ComBand}\xspace}
\newcommand{\ExpTwo}{\textsc{Exp2}\xspace}

\newcommand{\GeoHedge}{\textsc{GeometricHedge}\xspace}

\DeclareMathOperator{\reg}{Reg}	
\DeclareMathOperator{\swap}{SwapReg}	

\hypersetup{ %
    pdftitle={},
    pdfkeywords={},
    pdfborder=0 0 0,
    pdfpagemode=UseNone,
    colorlinks=true,
    linkcolor=Navy,
    citecolor=DarkGreen,
    filecolor=Purple,
    urlcolor=Black,
}

\usepackage[%
linewidth=2pt,
linecolor=gray,
middlelinecolor= black,
middlelinewidth=0.4pt,
roundcorner=1pt,
topline = false,
rightline = false,
bottomline = false,
rightmargin=0pt,
skipabove=0pt,
skipbelow=0pt,
leftmargin=0pt,
innerleftmargin=4pt,
innerrightmargin=0pt,
innertopmargin=0pt,
innerbottommargin=0pt,
]{mdframed}

\let\oldnl\nl
\newcommand{\nonl}{\renewcommand{\nl}{\let\nl\oldnl}}

\newcounter{protocol}


\DeclareMathOperator*{\argmax}{arg\,max}
\DeclareMathOperator*{\argmin}{arg\,min}

\usepackage{hyperref}
\usepackage{array}
\usepackage[utf8]{inputenc} 
\usepackage[T1]{fontenc}    
\usepackage{hyperref}       
\usepackage{url}            
\usepackage{booktabs}       
\usepackage{amsfonts}       
\usepackage{nicefrac}       
\usepackage{microtype}      
\usepackage{xcolor}         
\usepackage{amsmath}
\usepackage{amssymb}
\usepackage{mathtools}
\usepackage{amsthm}
\usepackage{enumitem}

\usepackage{multirow}

\theoremstyle{plain}
\newtheorem{theorem}{Theorem}[section]
\newtheorem{proposition}[theorem]{Proposition}
\newtheorem{lemma}[theorem]{Lemma}
\newtheorem{corollary}[theorem]{Corollary}
\newtheorem*{corollary*}{Corollary}
\theoremstyle{definition}
\newtheorem{definition}[theorem]{Definition}
\newtheorem{remark}[theorem]{Remark}
\theoremstyle{observation}

\usepackage[textsize=tiny]{todonotes}

\DeclarePairedDelimiter\ceil{\lceil}{\rceil}
\DeclarePairedDelimiter\floor{\lfloor}{\rfloor}

\usepackage{dsfont}

\title{Efficient Swap Regret Minimization in Combinatorial Bandits}

\author[1,3]{Andreas Kontogiannis\textsuperscript{*}}
\author[2,3]{Vasilis Pollatos\textsuperscript{*}}
\author[3,4]{Panayotis Mertikopoulos}
\author[3,5]{Ioannis Panageas}

\affil[1]{National Technical University of Athens, School of Electrical and Computer Engineering}
\affil[2]{National and Kapodistrian University of Athens, Department of Mathematics}
\affil[3]{Archimedes, Athena Research Center, Greece}
\affil[4]{Univ. Grenoble Alpes, CNRS, Inria, Grenoble INP, LIG, 38000 Grenoble, France}
\affil[5]{University of California, Irvine}

\begin{document}

\maketitle
\renewcommand{\thefootnote}{\fnsymbol{footnote}}
\footnotetext[1]{Equal contribution. Corresponding authors andreaskontogiannis@mail.ntua.gr, vaspoll@math.uoa.gr.}

\setcounter{footnote}{0}
\renewcommand{\thefootnote}{\arabic{footnote}}

\begin{abstract}
This paper addresses the problem of designing efficient no-swap regret algorithms for combinatorial bandits, where the number of actions $N$ is exponentially large in the dimensionality of the problem. 
In this setting, designing efficient no-swap regret translates to sublinear -- in horizon $T$ -- swap regret with polylogarithmic dependence on $N$. 
In contrast to the weaker notion of external regret minimization \textendash\ a problem which is fairly well understood in the literature \textendash\ achieving no-swap regret with a polylogarithmic dependence on $N$ has remained elusive in combinatorial bandits.
Our paper resolves this challenge, by introducing a no-swap-regret learning algorithm with regret that scales polylogarithmically in $N$ and is tight for the class of combinatorial bandits. To ground our results, we also demonstrate how to implement the proposed algorithm efficiently -- that is, with a per-iteration complexity that also scales polylogarithmically in $N$ -- across a wide range of well-studied applications.
\end{abstract}

\newpage

\section{Introduction} \label{sec: introduction}

\paragraph{Background on adversarial combinatorial bandits.}
Throughout this paper, we consider a general class of combinatorial online learning problems that unfold as follows:
\begin{enumerate}
\item
At each day (or round) $t = 1, 2, \dots, T$, of a repeated decision process, the learner selects a specific combination of up to $m$ elements from some ground set with $d$ elements in total.
\item
This choice triggers a reward for each element in the ground set; subsequently the learner receives as payoff the sum of the rewards of the chosen elements, and the process repeats.
\end{enumerate}
If the sequence of reward vectors is arbitrary rather than drawn from a stationary distribution, the problem is known as the \emph{adversarial combinatorial bandit problem} \cite{cesa2012combinatorial,LS20}.
Owing to its flexibility, this learning paradigm models a wide array of applications, from path planning \citep{KV05,BEL06,gyorgy2007line}, to resource allocation, recommender systems, and beyond \cite{kale2010non,audibert2014regret,WAEK15}.
Crucially, the total number of actions $N$ available to the learner typically grows \textit{exponentially} in the number of elements of the ground set; i.e., $N = \mathcal{O}(d^m)$.

This ``curse of dimensionality'' has led to extensive research, resulting in several algorithms with optimal (or near-optimal) guarantees for \emph{external} regret \textendash\ the cumulative gap between the learner's reward and the best fixed choice in hindsight.
In general, this literature revolves around two main axes: (a) the optimality of the regret guarantees provided in terms of $d$, $m$, and $T$, and (b) the per-iteration complexity of the methods proposed to achieve said guarantees.


To give an idea of the guarantees and techniques employed in this setting,
the original \ComBand algorithm of \cite{cesa2012combinatorial} \textendash\ also known in slightly different contexts as \GeoHedge \cite{dani2007price} or \ExpTwo \cite{bubeck2012towards} \textendash\ enjoys an external regret bound of $\widetilde{\mathcal{O}}(\sqrt{mdT})$ in specific combinatorial bandit problems, such as $m$-sets, path planning, spanning trees, perfect matchings, etc.
In general, the per-iteration complexity of \ComBand can be exponential, but it can be implemented efficiently in many settings with an amenable combinatorial structure \textendash\ using for example the kernelization approach of \cite{kontogiannis2025efficient}.
In a similar vein, to resolve implementation issues having to do with sampling from a low-support distribution,  \cite{kale2010non,combes2015combinatorial} proposed an approach based on Carathéodory decomposition techniques, in order to perform efficient updates in $d$ dimensions instead of $N$.

\textbf{From external to swap regret: motivation and challenges.}
Driven in no small measure by applications to game theory and reinforcement learning, there is a strong push in the literature toward \textit{more powerful} notions of regret \textendash\ such as internal and, more importantly for our purposes, \emph{swap regret}.
There are two main reasons for this:
First, from an online learning standpoint, external regret is not a particularly informative notion in environments where a fixed action policy severely underperforms relative to more reactive benchmarks. 
Second, from a multi-agent perspective, no-external-regret learning is not rationalizable because of well-known examples of negative external regret policies in which every player plays a strictly dominated strategy at all times \cite{VZ13}.

Both concerns above are mitigated by the notion of \emph{swap regret}, so dubbed by \cite{blum2007externalMansour}, and tracing its roots to the internal regret considerations already present in the original work of \cite{FV97} on calibrated forecasting.%
\footnote{The main difference between internal and swap regret is that the latter allows multiple strategy swaps at the same time.
We focus on the latter because of the added generality.}
Swap regret quantifies the additional utility a learner could have achieved, in hindsight, by retroactively modifying her played strategies according to the \textit{best fixed swap function} as opposed to the best fixed action in hindsight of the external regret.



On the downside, swap regret minimization is a significantly more challenging problem than vanilla, external regret minimization.
In his PhD thesis, \cite{stoltz2005thesis} proposed an algorithm with a swap regret bound of $\mathcal{O}(N \sqrt{T \log N})$, albeit with a per-iteration complexity of $\exp (\mathcal{O}(N))$.
At around the same period, \cite{blum2007externalMansour} proposed a $\poly(N)$-time algorithm for the full information setting which achieves a swap regret bound of $\mathcal{O}(\sqrt{NT\log N})$, using a swap-to-external regret reduction argument.
\cite{blum2007externalMansour} also extended their algorithm to the bandit setting, achieving a slightly worse regret bound of $ \mathcal{O}(N \sqrt{NT \log N})$.
This bound was later improved by \cite{ito2020tight} in the bandit setting, leading to an algorithm with a $\poly(N)$ per-iteration complexity that exhibits $\mathcal{O}(N \sqrt{T})$ swap regret bound.

One of the most challenging questions in the online learning community is how to design computationally efficient, no-swap-regret algorithms for which both the swap regret and the per-iteration complexity depend polylogarithmically on $N$. This was highlighted as an open question in \cite{blum2007externalMansour}. In the {full information setting}, the question was resolved very recently by \cite{daskalakis_swap, aviad_swap_regret}, with both approaches obtaining a swap regret bound of ${\mathcal{O}}(T/\log T)$. 
In the bandit setting, without assuming any structure, existing lower bounds for the bandit setting (e.g., see the lower bound for external regret in \cite{auer2002nonstochastic}, which carry over to swap regret) preclude polylogarithmic dependence on $N$.
The approach of \cite{daskalakis_swap} was also applied to the unstructured bandit setting, but, similarly to previous works, the resulting algorithm is impractical for large action spaces, as both the regret bound and the per-iteration complexity are polynomial in $N$. 

Although the aforementioned approaches are applicable to the combinatorial bandit setting, they fail to achieve the desired polylogarithmic dependence of the swap regret and the per-iteration complexity on $N$.
In this paper, we seek to examine whether the \textit{extra structure present in combinatorial bandits} can be exploited to yield a positive answer to the following question:

\begin{quote}
    \textit{Can we design efficient no-swap-regret learning algorithms for combinatorial bandits, achieving swap regret guarantees and per-iteration complexity that scale  polylogarithmically in the exponentially large number of actions $N=\mathcal{O}(d^m)$?}
\end{quote}

\vspace{-4.4pt}

\begin{table*}[t]
\colorlet{hgh}{green!20}
\centering
\renewcommand{\arraystretch}{1.5} 
\setlength{\tabcolsep}{4.4pt}
\begin{tabular}{lcccl}

\hline
\centering\textbf{Algorithm} & \centering \textbf{Swap Regret} & \textbf{Practical Bound} & \textbf{Per-Iteration} \\
\hline
\centering \cite{stoltz2005thesis} & \centering $\mathcal{O}\left(d^m\sqrt{mT \log d}\right)$ & $\text{poly}(d^m)$ & $\text{exp}(d^m)$ \\

\centering\cite{blum2007externalMansour} & \centering $\mathcal{O}\left(d^m\sqrt{md^m T \log d}\right)$ & $\text{poly}(d^m)$ & $\text{poly}(d^m)$ \\

\centering \cite{ito2020tight} & \centering $\mathcal{O}\left(d^m\sqrt{T}\right)$ & $\text{poly}(d^m)$ & $\text{poly}(d^m)$ \\

\centering \cite{daskalakis_swap} & \centering$\mathcal{O}\left(\frac{T\log(m\log T)}{\log(T/d^m)} +\sqrt{d^mT\log{(d^mT)}}\right){\color{blue}^*}$ & $\text{poly}(d^m)$ & $\text{poly}(d^m,\log T)$ \\

\centering\color{FocusColor}{\textbf{This work}} & \centering \cellcolor{hgh} $\mathcal{O}\left(\frac{T\log(d\log T)}{ \log(T)}\right)$ & \cellcolor{hgh} $\text{poly}(d,m)$ &  $\text{poly}(m,d,\log T)$ \cellcolor{hgh} \\
\hline
\centering{{Lower Bound}}\color{FocusColor}{$^{\dagger}$} & \centering  $\qquad$ ${\Omega\left(\frac{T}{\log^{6} T}\right)}$,$\ $ for $T \leq \exp(\Omega(d^{1/14}))$ &  \\
\hline
\end{tabular}
\caption{Comparative analysis of results on swap regret for combinatorial bandits and algorithms' per-iteration complexity. Practical bound denotes the swap regret in the practical regime where $T = \text{poly}(d,m)$. ${\color{blue}*}$:  \cite{daskalakis_swap} considers a stronger notion of swap regret, where the $\max$ operator is inside expectation, but it is also inefficient as it has $\mathcal{O}(N)$-per-iteration complexity. {\color{FocusColor}{${\dagger}$}}: The lower bound is a direct application of Theorem 4.1 in \cite{daskalakis2024lower} (see Corollary \ref{cor: lower_bound}). It holds for $m\geq d^{12/13}$.} 
\label{table: results}
\end{table*}

\textbf{Our contributions.} 
Our paper answers the above question affirmatively.
Specifically, we provide the first algorithm that achieves no-swap-regret with polylogarithmic dependence on the large action space of combinatorial bandits. Our main contribution is summarized in the following theorem.


\textbf{Theorem} (Abridged; Formally stated in Theorem \ref{thm: swap_lazy_comband_combcp})\textbf{.}
\textit{There exists a combinatorial bandit algorithm whose swap regret is at most}    $\mathcal{O}\left({T\log(d\log T)}/{\log T}\right).$

A detailed comparison with prior works is presented in Table \ref{table: results}. 
Importantly, only our swap regret upper bound remains meaningful in the realistic regime where $T =\text{poly}(d,m)$.
Remarkably, via an application (see Corollary \ref{cor: lower_bound}) of the lower bound result of \cite{daskalakis2024lower}, in the above regime of interest, our upper bound is tight, in the sense that there is no $T^p$-no-swap-regret scheme (where $p$ is a constant less than 1) with polylogarithmic dependence on $d^m$. 
In addition, using well-established techniques from combinatorial bandits, our algorithm can be implemented efficiently requiring time $\text{poly}(d,m)$ in a wide array of well-known applications, including online shortest paths, m-sets, spanning trees and permutations.

\paragraph{Technical Overview.}
Our approach is inspired by the results of \cite{daskalakis_swap,aviad_swap_regret}, which provide a reduction from swap to external regret for the full information setting that avoids introducing a polynomial dependence of the regret on the number of actions. 
Our algorithmic framework (Algorithm \ref{alg:multiscale_algo}) utilizes a master learner (that serves as the online learner of the combinatorial bandit setting) which is a uniform mixture over parallel learners, dubbed ScaleLearners.
Each ScaleLearner deploys an abstract combinatorial bandit algorithm, dubbed \textsc{Lazy-CombAlg}, and is parameterized by a different scale of laziness (that is, the learners only update their policies periodically and not at each day as is typical).
In our framework, \textsc{Lazy-CombAlg} is effectively a lazy version of a no-external-regret combinatorial bandit algorithm.
We propose an implementation of Algorithm \ref{alg:multiscale_algo} via an efficient implementation for \textsc{Lazy-CombAlg}, namely \textsc{Lazy-ComBCP} (Algorithm \ref{alg:lazy_combcp}), which is based on barycentric spanners and Carathéodory decomposition.

In contrast to the approaches of \cite{daskalakis_swap, aviad_swap_regret} in the full information setting, where the learning process is deterministic, we must contend with the stochasticity inherent in the partial information setting. 
In our framework, the master only observes bandit reward feedback. 
Based on this limited feedback, the master constructs an unbiased reward estimator. 
This same estimator is broadcasted to the ScaleLearners in order to be manipulated as their own reward estimator. 
The learning process follows an \textit{idiomatic bandit online learning protocol} where each ScaleLearner updates its policy using the received reward estimators, coming from the master, which are \textit{unbiased} w.r.t the master's policy but \textit{biased} w.r.t. the ScaleLearner's policy. 
Crucially, the above protocol allows us to decompose the swap regret of the master learner into a sum of external regrets of the deployed \textsc{Lazy-CombAlg} instances of the ScaleLearners (see Lemma \ref{lem: swap_regret_}). 
The main challenge in bounding the external regret of each learner is that the stochasticity of the expectation of the regret is w.r.t. the master's policy, rather than the learner's as would be the standard.  
In the standard analysis of no-external-regret combinatorial bandit algorithms (e.g., \cite{cesa2012combinatorial}), one often needs to control a term of the form $\mathbb{E}[ M^T \Sigma^{+} M]$ \footnote{$M$ denotes the combinatorial action vector and $\Sigma^{+}$ denotes the pseudo-inverse of the \textit{co-occurrence matrix} $\mathbb{E}[MM^T]$, see also Algorithm \ref{alg:multiscale_algo}.} in order to bound the variance of the reward estimator. 
Interestingly, in contrast to the above, to show no external regret of each learner (Theorems \ref{thm:combcp_no_regret_}) in our framework, it turns out that we need to control $\mathbb{E}[ M^T \Sigma^{+2} M]$ instead (see Lemma \ref{lem:combcp_bounded_variance}), which requires to carefully control the inverse of the minimum eigenvalue of the induced co-occurrence matrix $\mathbb{E}[MM^T]$ (see Lemma \ref{lem:key_}). 
Finally, we show that a regret upper bound of $\mathcal{O}(T\log(d\log T)/\log T)$ can yield anytime guarantees (Lemma \ref{lem:anytime_}) for any online learning algorithm achieving this bound -- including the proposed algorithms.

\section{Preliminaries} \label{sec: prelims}

\textbf{Combinatorial Bandits.}
We consider the classic adversarial combinatorial bandit setting, proposed by \cite{cesa2012combinatorial}. In this online learning setting, the decision maker has access to a finite action set $\mathcal{A} \subset \{0,1\}^d$, that is a subset of the $d$-dimensional hypecube. Specifically, for a given $m \leq d$, any action $M \in \mathcal{A}$ is a binary vector with at most $m$ ones, that is, $\|M\|_1 \leq m$. 

The online learning setting is as follows: The decision maker participates in a repeated process over days, where at each day $t = 1,2,\ldots,T$, she has a distribution (denoted by policy) $p_t$ over $\mathcal{A}$ and samples from it an action $M_t$. Then, she receives from the environment a reward $r_t = R_t \cdot M_t \in [0,1]$, where $R_t$ is the reward vector for that day, with $R_t(j)$, for all $j \in [d]$, being the reward for the $j$-th coordinate. In this paper, we assume that the rewards are selected by an oblivious adversary, that is a type of opponent or environment that chooses the entire sequence of rewards in advance, before the learning algorithm starts. This means the adversary does not adapt to the actions or internal randomness of the learning algorithm.

\textbf{Notions of Regret.}
The most standard and well-studied objective for the combinatorial bandit setting is the \textit{external regret} which measures the difference between the algorithm’s cumulative reward and the best fixed action in hindsight. Formally, for a given horizon of $T$ days, the external regret is as follows:
\begin{align}
    \reg_T = \max_{M^* \in \mathcal{A}} \sum_{t=1}^T R_t \cdot M^* - \mathbb{E}\left[ \sum_{t=1}^T R_t \cdot M_t \right] .
\end{align}
Although the bound on external regret is crucial, it may be less appealing in highly dynamic environments where no single action consistently performs well throughout the entire history of rewards. 

In this paper, we study \textit{swap regret}, a notion of regret stronger than the external. Swap regret compares the decision maker's policy $p_t$ against all strategies that can be derived from $p_t$ 
by applying a swap function $\phi: \mathcal{A} \rightarrow \mathcal{A}$ to actions sampled from $p_t$. Formally, we define $\Phi$ to be the set containing all swap functions that map from $\mathcal{A}$ to $\mathcal{A}$. The swap regret is defined as follows:
\begin{align}\label{swap_regret}
\swap_T = \max_{\phi \in \Phi} \mathbb{E}\left[\sum_{t=1}^T \bigg( R_t \cdot \phi(M_t) -  R_t \cdot M_t \bigg) \right].
\end{align}

\section{\textsc{Swap-ComBCP}: A No-Swap-Regret Combinatorial Bandit Algorithm}\label{sec: multi-scale}

In this section, we present our algorithmic approach.
Typically, to bound swap regret, one first reduces it to external regret and then minimizes the latter (e.g., \cite{blum2007externalMansour, daskalakis_swap}).
In this work, for the purposes of the reduction from swap to external regret, we use a backbone algorithmic framework for combinatorial bandits (see Algorithm \ref{alg:multiscale_algo}) which achieves a decomposition of swap regret into a sum of external regret terms, as we show in Lemma \ref{lem: swap_regret_}.
Then, to derive a vanishing swap regret guarantee, we need to control these external regret terms.
Typically, in a swap-to-external-regret reduction, bounding the resulting external regret terms is a straightforward task; however, in our setting it poses significant challenges, which will be discussed later in detail.
We combine Algorithm \ref{alg:multiscale_algo} with an efficient combinatorial bandit algorithm, namely \textsc{Lazy-ComBCP} (Algorithm \ref{alg:lazy_combcp}), and bound the external regret of the latter in Theorem \ref{thm:combcp_no_regret_}.
With this bound at hand, we get no-swap-regret guarantees for the overall algorithm (which we denote by \textsc{Swap-ComBCP}), stated in our main result (Theorem \ref{thm: swap_lazy_comband_combcp}).

\textbf{Multi-Scalar Laziness Against Swap Regret.}
Swap regret is a difficult solution concept to handle, because the reference swap function $\phi$ may exploit patterns that arise due to the variability of the decision maker's policy over time. 
Thus, one way to attack swap regret is by reducing the variability of the policy.
In the extreme case, where the policy remains fixed during the whole horizon of play, swap deviations can achieve no higher utility than deviations to a fixed action, that is swap regret is upper bounded by the external regret. 
We state this formally in the following proposition.
\begin{proposition}[Swap Regret is upper bounded by External Regret for time-invariant policies]\label{property_laziness}
    For any policy $p\in\Delta(\mathcal{A})$ fixed in an interval $t=1, \ldots, T$, any sequence of rewards $R_t$ and any swap function $\phi: \mathcal{A}\rightarrow \mathcal{A} $ it holds that
    \begin{align}
    \sum_{t=1}^T\sum_{M \in \mathcal{A}}p(M) R_t\cdot \phi(M)\leq \max_{M \in \mathcal{A}} \left\{ \sum_{t=1}^T R_t\cdot M \right\}.
    \end{align}
\end{proposition}

Of course, keeping the policy fixed during the whole horizon would potentially result in extremely poor performance due to the lack of adaptivity. Thus, to exploit the above property we have to keep a balance between stability and adaptivity. For this purpose, as in \cite{daskalakis_swap, aviad_swap_regret}, we use the concept of a \textit{lazy} online learner, that is a learner that does not update its policy at every time-step but instead it only updates periodically, and keeps its policy freezed in the time period between two consecutive updates.

But again, a single lazy external regret learner is not enough to achieve vanishing swap regret against any adversary. To keep the right \textit{balance in laziness}, one would need to use multiple learners, each with a different scale of laziness, and achieve the necessary balance by creating an \textit{ensemble} of them. Lazy learners that freeze longer observe the learning task at a higher scale, thus we refer to them as ScaleLearners and the ensemble algorithm as the multi-scale master learner.


\begin{algorithm*}[t]
\caption{Master Multi-scale Algorithm}\label{alg:multiscale_algo}
\begin{algorithmic}[1]
\STATE \textbf{Require}: $T$, sequence of rewards $R_1, \dots, R_T$ \quad $\mid$ \quad \textbf{Parameters}: $K$, $H$ such that $T \in [H^{K-1}, H^{K}]$ 
\FOR{$t=1,2,...,T$}
\STATE $\text{Let } p_{k, t} \in \Delta^{|\mathcal{A}|} \text { be the policy of the}$ $k$-th {\color{MyGreen}{\textsc{ScaleLearner}$_k$}}  ($k \in\left[K\right])$
\STATE $\text{Let } \widehat{p}_{t} \in \Delta^{|\mathcal{A}|}$ \text { be the policy of the master learner that plays uniformly over } \\ ScaleLearners; that is: 
$\quad \widehat{p}_{t}=\frac{1}{K} \sum\limits_{k \in\left[K\right]} p_{k, t}$
\STATE \emph{Sampling:} Play action $M_t \sim \widehat{p}_{t}$ and receive reward $r_t=R_t \cdot M_t$ {\color{blue} \qquad  /$^*$ $R_t$ is unobserved $^*$/ }
\STATE \emph{Estimation:} Let $\Sigma_{t}=\mathbb{E}\left[ MM^{\top}\right]$, where $M$ has law $\widehat{p}_{t}$. Set $\widetilde{R}_t = r_t\Sigma_{t}^{+}M_t$
\STATE \textit{Broadcast
} $(t, \widetilde{R}_t)$ to all \textsc{ScaleLearner}$_k$, for all $k \in [K]$ 
\ENDFOR

\STATE{\text{}}
\\
\STATE {\textbf{Procedure}} {\color{MyGreen}{\textsc{ScaleLearner}$_k$}} \\
\STATE \textbf{Require}: \textsc{Lazy-CombAlg}  {\color{blue} \qquad  /$^*$ A lazy combinatorial bandit algorithm (e.g., Algorithm \ref{alg:lazy_combcp}) $^*$/}
\FOR{$l=1,2, \ldots, \ceil*{{T}/{H^k}}$} 
\STATE {\textit{Restart}}: initialize the parameters of \textsc{Lazy-CombAlg}$_{k,l}$  
\STATE \textit{Run} \textsc{Lazy-CombAlg}$_{k,l}$ {\color{blue} \qquad  /$^*$ Play for $H^k$ days, policy updated every $H^{k-1}$ days $^*$/}
\ENDFOR
\end{algorithmic}
\end{algorithm*}

\subsection{The algorithmic framework}
Our framework (Algorithm \ref{alg:multiscale_algo}) consists of the master learner $(\text{with policy }\hat{p}_t)$ which coordinates \textit{parallel} thread learners, referred to as ScaleLearners, all of which share the global variable $T$. 
Each ScaleLearner thread (dubbed \textsc{ScaleLearner}$_k$) is parameterized by a different scale $k$ of laziness. 
For each $k \in [K]$, the corresponding \textsc{ScaleLearner}$_k$ divides the online learning horizon $T$ into $\ceil*{\frac{T}{H^k}}$ equal intervals and restarts its policy at the beginning of each interval (Step 13). 
At each interval $(k,l)$, where $k\in [K]$ and $l \in \left\{1,\dots,\ceil*{\frac{T}{H^k}}\right\}$, \textsc{ScaleLearner}$_k$ deploys \textsc{Lazy-CombAlg$_{k,l}$}, which -- for now -- serves as an abstract combinatorial bandit learner. We describe it in detail in Section \ref{sec: combcp}. 
The common feature among all \textsc{Lazy-CombAlg$_{k,l}$} learners is that they each perform $H$ updates. 
What distinguishes them is how many days they parse between consecutive updates.
In particular, \textsc{Lazy-CombAlg}$_{k,l}$ parses $H^{k-1}$ days between two consecutive updates staying lazy in between. We denote this period by the term \textit{meta-day}.
Based on this, parameter $k$ is used to control the level of laziness of the ScaleLearners, with the master policy being a uniform mixture of the policies of the ScaleLearners (Step 4), aiming to balance between laziness and no-regret learning.   

The above steps were also used in the reduction of \cite{daskalakis_swap, aviad_swap_regret} for the full information setting. The key difference here is that we need to deal with the stochasticity of the combinatorial bandit setting. 
In our framework, the master samples an action $M_t \sim \hat{p}_t$ and observes a bandit reward $r_t$ (Step 5). 
In contrast to the full information setting where the reward feedback is deterministic and given for all available actions, in our setting it is not straightforward what reward feedback the ScaleLearners should utilize.
Following a well-established technique from combinatorial bandits (e.g., \cite{cesa2012combinatorial, dani2007price}), the master computes the reward estimate $r_t(\mathbb{E}_{\hat{p}_t}[MM^T])^{+}M_t$ (Step 6), which is unbiased w.r.t. to its own policy $\hat{p}_t$.
Then, the master broadcasts this estimate to the ScaleLearners of the ensemble (Step 7).
Notably, the above stand in stark contrast to the approach  in the bandit algorithm of \cite{daskalakis_swap}, where the reward estimator is constructed based on the policy of the corresponding ScaleLearner. 


Most importantly, our algorithmic framework entails an \textit{idiomatic bandit learning protocol}:
In contrast to the standard online learning protocol, ScaleLearners do not actually play any action, they only have a policy distribution over actions at each time-step. 
Each ScaleLearner updates its policy using the received reward estimators, coming from the master, which are \textit{unbiased} w.r.t. the master's policy but \textit{biased} w.r.t. the ScaleLearner's policy. 
This protocol allows us to derive a decomposition of the master's swap regret into the external regrets of the lazy learners through the following lemma.


\begin{lemma}[Swap Regret Decomposition under Bandit Feedback]\label{lem: swap_regret_}
    For any $\phi \in \Phi$, the swap regret of Algorithm \ref{alg:multiscale_algo} can be bounded as follows:
    \begin{align}
    \swap_{T}({\phi})  \leq & \frac{1}{K}\sum_{k=1}^{K-1}\sum_{l=1}^{\ceil*{\frac{T}{H^{k}}}}\reg(\textsc{Lazy-CombAlg}_{k,l}) +\frac{T}{K} . \nonumber
    \end{align}
\end{lemma}

\begin{remark}
    The role of the restarting period, $H^k$, is pivotal for the analysis in order to decompose the swap regret of Algorithm \ref{alg:multiscale_algo} into the individual external regrets of each learner. 
    For the interested reader, we provide an intuitive proof sketch illustration of Lemma \ref{lem: swap_regret_} in Fig. \ref{fig_tree_swap} (Appendix \ref{appendix: proof_tree_swap}).
\end{remark}

\begin{remark}
It is worth noting that beyond the laziness part, a key challenge in bounding the external regret of \textsc{Lazy-CombAlg}$_{k,l}$ is that the stochasticity of the expectation in the external regret is under the law of the master policy $\hat{p}_t$, rather than the \textsc{Lazy-CombAlg}$_{k,l}$'s policy $p_{k,t}$ as would be the standard. We discuss the above in more detail in Section \ref{sec: combcp}.
\end{remark}


Based on Lemma \ref{lem: swap_regret_}, our ultimate goal will be to design \textsc{Lazy-CombAlg} such that the external regret terms appearing in Lemma \ref{lem: swap_regret_} are vanishing, thereby achieving a no-swap-regret guarantee for Algorithm \ref{alg:multiscale_algo}. 
A key design goal is also to ensure that the per-iteration complexity of the overall framework scales polynomially with respect to $d$ and $m$.

\subsection{Instantiating \textsc{Lazy-CombAlg}}\label{sec: combcp}

\begin{algorithm*}[t]
\caption{\textsc{Lazy-ComBCP$_{k,l}$}}\label{alg:lazy_combcp}
\begin{algorithmic}[1]
\STATE Compute a $2$-\textit{approximate barycentric spanner}, $S$, of $\mathcal{A}$, and let $\mu = \frac{1}{d}  \mathds{1} 
\{M \in S\}$
\STATE \textit{Initialize} $q_{k,l,1} = [1/d, ..., 1/d] \in \Delta^d$, $\gamma=H^{-1/3}$, $\eta=\frac{1}{d^3 \sqrt{m} H^{k-1/3}}$
\STATE 
    $H' = \begin{cases}\qquad H &\text{ if }  l\leq\frac{T}{H^k} \quad \quad \text{{\color{blue} /$^*$Play for $H$ meta-days,}} \\
    \ceil*{\frac{T-\floor*{{T}/{H^k}}H^k}{H^{k-1}}}  &\text{ otherwise} \qquad \text{{\color{blue}or until the time limit $T$ is reached$^*$/}}
    \end{cases}$ 
\FOR{$h=1,2, \ldots, H $}
\STATE \emph{Carathéodory Decomposition:} 
$mq_{k,l,h} = \sum_{M \in \mathcal{M}} \widetilde{p}_{k,l,h}(M) M$
\STATE \emph{Mixing:} Let $p_{k,l,h}=(1-\gamma)\widetilde{p}_{k,l,h}+\gamma\mu$ 
\STATE {\textit{Meta-day}:} $\text{Fix } p_{k,l,h} \text { for } H^{k-1} \text { days and aggregate the rewards of these days}$ 
$$\widetilde{X}_{k,l,h}(i)=\sum_{\tau=(\ell-1) H^k+(h-1) H^{k-1}+1}^{\min((\ell-1) H^k+h H^{k-1},T)} \widetilde{R}_{\tau}(i), \quad {\forall i \in[d]}$$
\STATE \emph{Update:}  
$\widetilde{q}_{k,l,h+1}(i) = q_{k,l,h}(i) \exp(\eta \widetilde{X}_{k,l,h}(h)),\; \forall i\in[d]$
\STATE \emph{Projection:} 
$q_{k,l,h+1} = \argmin_{q \in \mathcal{P}} \text{KL}(q, \widetilde{q}_{k,l,h+1}) $
\ENDFOR
\end{algorithmic}
\end{algorithm*}

In this section, we present our no-external-regret combinatorial bandit algorithm, namely \textsc{Lazy-ComBCP} (see Algorithm \ref{alg:lazy_combcp}), which can efficiently implement the $\textsc{Lazy-CombAlg}$ learner in Algorithm \ref{alg:multiscale_algo}.
We denote the corresponding overall framework (i.e., Algorithm \ref{alg:multiscale_algo} combined with Algorithm \ref{alg:lazy_combcp}) by \textbf{\textsc{Swap-ComBCP}}.

Before discussing the \textsc{Lazy-ComBCP} algorithm, we will introduce the following useful notation. Let $p_{k,l,h}$ be the policy of $\textsc{Lazy-ComBCP}_{k,l}$ at the $h$-th meta-day of its interval of play, that is from $t=(\ell-1) H^k+(h-1) H^{k-1}+1$ to $t=\min((\ell-1) H^k+h H^{k-1},T)$. 
Moreover, let $M_{k,l}^* \in \argmax_{M \in \mathcal{A}} \sum_{h=1}^{H} X_{k,l,h} \cdot M$; i.e. $M_{k,l}^*$ is an optimal action in hindsight for $\textsc{Lazy-ComBCP}_{k,l}$, within the interval $l$, for $H$ meta-days and bandit feedback $X_{k,l,h}$ at meta-day $h$. Due to space constraints, when the indices $k,l$ are fixed, that is when we analyze the external regret of \textsc{Lazy-ComBCP}$_{k,l}$ learner, we sometimes abuse the notation, and we write $\widetilde{p}_h=\widetilde{p}_{k,l,h}$, $\widetilde{X}_{h} = \widetilde{X}_{k,l,h}$, $M_{\tau} = M_{t}$, $\Sigma_{\tau} = \Sigma_{t}$ and $\mathbb{E}_{k,l,h}[\cdot]=\mathbb{E}_t[\cdot] =\mathbb{E}[\cdot|\mathcal{F}_t]$ where $t=(\ell-1) H^k+(h-1) H^{k-1}+\tau$.


Now, we are ready to present $\textsc{Lazy-ComBCP}_{k,l}$ (Algorithm \ref{alg:lazy_combcp}). 
We adopt barycentric spanners (\textbf{\textbf{B}}), Carathéodory decomposition (\textbf{C}) and projection (\textbf{P}) techniques as key ingredients in our algorithm. 
Due to space constraints, background for the notion of barycentric spanners can be found in Appendix \ref{sec: barycentric}. 
Algorithm \ref{alg:lazy_combcp} assumes that $\|M\|_1 = m$, for any $M \in \mathcal{A}$ (see also the discussion in Section \ref{sec: applications} regarding the more general case where $\|M\|_1 \leq m$).

As suggested by our algorithmic  framework (Algorithm \ref{alg:multiscale_algo}), $\textsc{Lazy-ComBCP}_{k,l}$ operates lazily over $H'$ meta-days, where 
$H'$ is equal to $H$ in general, unless \textsc{ScaleLearner}$_k$ is performing its final restart and $H^k$ does not perfectly divide $T$ (Step 3).
Each full meta-day consists of $H^{k-1}$ consecutive days. 
The learner performs updates on the aggregated reward estimates $\widetilde{X}$ (which have been broadcasted from the master) only at the end of each meta-day, staying lazy in between.

More specifically, our algorithm operates in the coordinate space $\mathbb{R}^d$ as follows: 
Let $\mathcal{M}$ be the convex hull of $\mathcal{A}$, and let $\mathcal{P}$ be the scaled convex hull of $\mathcal{A}$; that is, $\mathcal{P}$ is the convex hull of the actions $M \in \mathcal{A}$ divided by $m$. 
The algorithm uses a distribution ${q} \in \Delta^{d}$, which is updated (Step 8) based on the aggregated reward estimates (Step 7). 
In order to guarantee that the updated distribution is in $\mathcal{P}$, we project it onto $\mathcal{P}$ using the KL divergence (Step 9).
Based on the definition of $\mathcal{P}$, it holds that $mq \in \mathcal{M}$, and thus $mq$ can be decomposed (Step 5) into an efficient distribution, $\widetilde{p} \in \Delta^{|\mathcal{A}|}$, with small support (of at most $d$ actions).
To efficiently explore all $d$ coordinates, the learner's policy is a mixture of $\widetilde{p}$ with an exploration policy $\mu$, which is the uniform distribution over the elements of an approximate barycentric spanner over $\mathcal{A}$ (Step 6). 

As the next lemma shows, the use of the barycentric spanner allows us to lower bound the minimum eigenvalue of $\Sigma_{\tau}$, which is of great importance in order to control the boundedness and the variance of the aggregated reward estimates $\widetilde{X}$.


\begin{lemma}\label{lem: 2016_lem_}
    Let $\lambda_{\text{min}}(\mu)$ be the minimum eigenvalue of the co-occurrence matrix, $\Sigma_{\tau}$, under the law of the barycentric spanner exploration policy $\mu$. It holds that
    $ \lambda_{\text{min}}(\mu) \geq \frac{1}{4 d^3}.
    $
\end{lemma}

\begin{remark}[Efficient Sampling]
   The use of the barycentric spanner, along with the decomposition step, is crucial for efficient sampling (Step 4 in Algorithm \ref{alg:multiscale_algo}). Both the exploration policy $\mu$ and the  policy $\tilde{p}_{k,l,h}$ have small support of size at most $\Theta(d)$, thus allowing efficient sampling from the ScaleLearner's policy $p_{k,t}$, and subsequently from the master learner's policy $\hat{p}_{t}$ (see Algorithm \ref{alg:multiscale_algo}, Step 4).  
\end{remark}

\begin{remark}[Mixing step]
    Our mixing step in Algorithm \ref{alg:lazy_combcp} fixes an issue of \textsc{CombExp} \cite{combes2015combinatorial}, which first decomposes the mixture of $q$ and the distribution over coordinates induced by the uniform distribution over $\mathcal{A}$.  
    In particular, in settings of interest where the coordinates are unevenly represented in $\mathcal{A}$ (e.g., paths), the exploration of \cite{combes2015combinatorial} may require a number of steps exponential in $d$ to compete with the optimal action in hindsight. 
    For a more detailed discussion, we refer the reader to Appendix \ref{appendix: combexp}.
\end{remark}

Next, we show two important properties of the aggregated reward estimate $\widetilde{X}$. 
As we show in the following lemma, the aggregated reward estimate is bounded and remains unbiased, importantly, with respect to the master’s policy, since the expectation's randomness stems from the master’s policy rather than the policy of \textsc{Lazy-ComBCP}$_{k,l}$.

\begin{lemma}\label{bcp: properties_}
For all $k \in \{1,...,K-1\}$, $l \in \left[\ceil*{\frac{T}{H^{k}}}\right]$ and $h \in [H]$, $\textsc{Lazy-ComBCP}_{k,l}$ satisfies the following properties:
\begin{enumerate}
    \item  (Unbiasedness): $ \mathbb{E}[q_{h} \cdot\widetilde{X}_{h}] = \mathbb{E}[q_{h} \cdot{X}_{h}] $
    \item (Boundedness): $ \|\widetilde{X}_{h}\|_{2} \leq \frac{4 H^{k-1} d^3 \sqrt{m}}{\gamma}$
\end{enumerate}
\end{lemma}


In the following theorem, we show that Algorithm \ref{alg:lazy_combcp} satisfies the no-external-regret property.

\begin{theorem}[No-External-Regret]\label{thm:combcp_no_regret_}
The external regret of $\textsc{Lazy-ComBCP}_{k,l}$ is at most \\ $3 H^{k-1} H^{2/3} d^3 m^{3/2} \log d$.  
\end{theorem}

We note that the above bound implies no regret, because the $H^{k-1}$ factor is the maximum reward per meta-day of the $\textsc{Lazy-ComBCP}_{k,l}$ learner.

\subsection{Proof Sketch of Theorem \ref{thm:combcp_no_regret_}}

For simplicity, we assume that $H' = H$. In the case where $H'<H$, the derived upper bound still holds, since the actual horizon of play is smaller. 
Let $p_{k,l}^*$ be an optimal deterministic policy of the adversary, with $p_{k,l}^*(M_{k,l}^*) = 1$, of the learner $k \in [K]$ for the interval indexed by $l \in \left[\ceil*{\frac{T}{H^{k}}}\right]$. 
First, we derive the following proposition.

\begin{proposition}[Decomposition of optimal deterministic policy]\label{prop: opt_}
    Let $q_{k,l}^* = \frac{M^*_{k,l}}{m}$. It holds that 
    $ m q_{k,l}^* =  \sum_{M \in \mathcal{A}} p^*_{k,l}(M) M. $
\end{proposition}

Using the unbiasedness of the estimator (Lemma \ref{bcp: properties_}) and Proposition \ref{prop: opt_}, we obtain the following proposition, which at first glance provides an upper bound on the external regret.

\begin{proposition}\label{prop: combcp_}
For all $k \in \{1,...,K-1\}$ and $l \in \left[\ceil*{\frac{T}{H^{k}}}\right]$, the external regret of $\textsc{Lazy-ComBCP}_{k,l}$ can be bounded as follows:
    \begin{align}
        \reg_{H} & \leq m\mathbb{E}\left[ \sum_{h=1}^{H}  q_{k,l}^* \cdot \widetilde{X}_{h} - \sum_{h=1}^{H} q_{h} \cdot \widetilde{X}_{h} \right] + \gamma H^{k-1}. \nonumber 
    \end{align}
\end{proposition}

Now, we need to bound the first term of the above expression. Using Lemma \ref{bcp: properties_}, we ensure that $\eta \|\widetilde{X}_h\|_{\infty} \leq 1$, so that \textsc{Lazy-ComBCP} satisfies the following Online Mirror Descent lemma.

\begin{lemma}[Online Mirror Descent on $\mathcal{M}$]\label{lem: omd_combcp_}
    If $\eta \|\widetilde{X}_{h}\|_{\infty} \leq 1$, it holds that
    \begin{align}
    \mathbb{E}\left[\sum_{h=1}^H \left(q_{k,l}^* - q_{h} \right)  \cdot \widetilde{X}_{h} \right]
    &\leq \eta \mathbb{E}\left[ \sum_{h=1}^H q_{h} \cdot \widetilde{X}_{h}^2\right] \nonumber  + \frac{\log d}{\eta}, \nonumber 
    \end{align} 
    where $\widetilde{X}_h^2$ is the coordinate-wise square of $\widetilde{X}_h$.
\end{lemma}

By decomposing $q_h$ on the variance term of the above expression, we obtain the following:  
\begin{align}
&\mathbb{E}_{k,l,h}\left[\sum_{M \in \mathcal{A}} q_{h} \cdot \widetilde{X}_{h}^2 \right] = \frac{1}{m}\mathbb{E}_{k,l,h}\left[\sum_{M \in \mathcal{A}} \widetilde{p}_{h}(M) \left(\sum_{\tau=1}^{H^{k-1}}\widetilde{R}_{\tau}\right)^2 \cdot M\right] \label{eq: variance_terms}
\end{align}

The challenge here is that the above variance term contains the {squared sum of the reward estimates} (which are due to the laziness of the algorithm), all of which are sampled by the \textit{master's policy} $\widehat{p}_{k,l,h,\tau}$, and not by the ScaleLearner's policy $\widetilde{p}_h$, the external regret of which we aim to bound.  
In contrast to the analysis of no-external-regret combinatorial bandit algorithms (e.g., \cite{dani2007price, bartlett2008high, cesa2012combinatorial, combes2015combinatorial}), where bounding the variance involves controlling \( \mathbb{E}_{k,l,h}\left[ M_{\tau}^T \Sigma_{\tau}^{+} M_{\tau} \right] \), it turns out that we need to bound \( \mathbb{E}_{k,l,h}\left[ M_{\tau}^T \Sigma_{\tau}^{+2} M_{\tau} \right] \) instead. The following lemma provides a bound for these expectations.

\begin{lemma}\label{lem:key_}
    For all $k \in \{1,...,K-1\}$, $l \in \left[\ceil*{\frac{T}{H^{k}}}\right]$ and $h \in [H]$, it holds that:
    $$ \mathbb{E}_{k,l,h}\left[ M_{\tau}^T \Sigma_{\tau}^{+2} M_{\tau} \right] \leq d \lambda^{-1}_{\text{min}}(\Sigma_{\tau}), $$
    where $\lambda_{\text{min}}(\Sigma_{\tau})$ is the minimum nonzero eigenvalue of $\Sigma_{\tau} = \mathbb{E}_{\hat{p}_{k,l,h,\tau}}[MM^T]$.
\end{lemma}

Now, using Lemmas \ref{lem:key_} and \ref{lem: 2016_lem_} and a careful analysis, we derive the following final lemma, which concludes the proof.

\begin{lemma}[Bounded variance]\label{lem:combcp_bounded_variance}
For all $k \in \{1,...,K-1\}$, $l \in \left[\ceil*{\frac{T}{H^{k}}}\right]$ and $h \in [H]$, $\textsc{Lazy-ComBCP}_{k,l}$ satisfies the following: 
    \begin{align}
        \mathbb{E}[q_{h} \cdot \widetilde{X}_{h}^2] &
        \leq H^{k-1} \sum_{\tau=1}^{H^{k-1}} \mathbb{E}_{k,l,h}\left[ M_{\tau}^T \Sigma_{\tau}^{-2} M_{\tau} \right] \leq \frac{4H^{2k-2} d^4}{\gamma} \nonumber
    \end{align}
\end{lemma}

\section{Main Result}\label{sec: main_result}

\subsection{Swap Regret Bounds}

Combining the no-external-regret guarantees of Algorithm \ref{alg:lazy_combcp} (Theorem \ref{thm:combcp_no_regret_}) with Lemma \ref{lem: swap_regret_}, we derive our main result, which provides the no-swap-regret guarantees of \textsc{Swap-ComBCP}.

\begin{theorem}[No-Swap-Regret]\label{thm: swap_lazy_comband_combcp}
    Setting $H=27\lfloor \log(T) \rfloor^3 d^9m^{9/2}\log^3 d$ and $K=\lfloor \log_H(T) \rfloor$, \textsc{Swap-ComBCP} satisfies the following swap-regret guarantee:
    $$ \swap \leq \frac{45\log (d\log T)T}{ \log T}.$$ 
\end{theorem}

To provide anytime guarantees, typically required for online learning, we next show that the Doubling Trick \cite{auer2002adaptive, besson2018doublingTrick} can be applied to our algorithm.

\begin{lemma}[Anytime Guarantees]\label{lem:anytime_}
Assuming that, for given $T$, Algorithm \ref{alg:multiscale_algo} achieves
$$\swap_{T}({\phi}) \leq \mathcal{O}\left(\frac{T\log(d\log T)}{\log T} \right),$$ 
then it achieves the above bound with anytime guarantees, if implemented with the Doubling Trick.
\end{lemma}

By applying the lower bound of \cite{daskalakis2024lower} in combinatorial bandits (see also Appendix \ref{sec: lower_bound}), we obtain the following corollary, which states that in the regime of interest, that is $T < \exp(d)$, our swap regret upper bound is tight, in the sense that there exists no $T^p$-no-swap-regret scheme, where $p$ is a constant less than 1.

\begin{corollary}[Tightness]\label{cor: lower_bound}
For $T \leq \exp(\Omega(d^{1/14}))$ there exists a combinatorial bandit setting with action set $\mathcal{A} \subset \{0,1\}^d$ and an oblivious adversary running for $T$ rounds against whom any learner suffers swap regret at least $\Omega\left({T}/{\log^{6} T}\right)$.
\end{corollary}

\subsection{Per-Iteration Complexity in Well-Studied Combinatorial Settings}\label{sec: applications}

We present some important combinatorial bandit settings where \textsc{Swap-ComBCP} is applicable (i.e., the $L_1$-norm of the action vectors equals a fixed value $m$) and efficiently implementable (i.e., we can efficiently perform decomposition, projection and approximate barycentric spanner calculation). 
The settings we consider are \textit{m-sets}, \textit{Spanning Trees and Graph Matroid Bandits (k-forests), Paths, Permutations} and \textit{Truncated Permutations}. 
For all the above settings, except for paths, efficient decomposition and projection algorithms are proposed in \cite{suehiro2012online}. 
For paths in DAGs, the path polytope has a succinct description which allows efficient decomposition and projection with interior point methods (see \cite{boyd2004convex}, page 545). 
A careful reader might observe that paths do not necessarily satisfy the condition  that $\|M\|_1$ can vary across $\mathcal{A}$.
Fortunately, it turns out that one can always transform the DAG, so that this condition is  satisfied (see \cite{gyorgy2007line}). 
Moreover, for all the above settings, we can efficiently compute an approximate barycentric spanner as there exists an efficient linear minimization oracle over $\mathcal{A}$ (see Proposition \ref{prop:barycentric_algo}).
More details are also provided in Appendix \ref{appendix: apps_combcp}.

Regarding the more general setting, where $\|M\|_1$ varies across actions $M \in \mathcal{A}$, we provide an alternative implementation for \textsc{Lazy-CombAlg}, namely \textsc{Lazy-ComBand}, which is a lazy version of the well-known \textsc{ComBand} algorithm \cite{cesa2012combinatorial}. Due to space constraints, we only include this algorithm in Appendix \ref{appendix: comband}. 
Both \textsc{Swap-ComBCP} and \textsc{Swap-ComBand} achieve the same asymptotic swap regret guarantees.
Although \textsc{Swap-ComBand} is applicable to the more general setting, the applications, where efficient implementations of it are known, are more limited compared to \textsc{Swap-ComBCP}. 
In particular, the main applications of \textsc{Swap-ComBand} are paths, as well as problems that can be modelled as paths, such as m-sets, both of which are also covered by \textsc{Swap-ComBCP}.  



\section{Conclusion}\label{sec: conclusion}

Our paper focused on the problem of no-swap-regret learning in combinatorial bandits, i.e., structured bandit settings where the number of actions is exponentially large (in the settings description). 
Importantly, we provided the first online algorithms for combinatorial bandits which achieve vanishing swap regret and per-iteration complexity that scale polylogarithmically in the number of actions. 
One interesting open question is whether we can obtain similar results in combinatorial bandits for realized swap regret with high probability guarantees. 



\section*{Acknowledgments}

Most of this work was done while all authors were visiting Archimedes Research Unit, Athens, Greece. Ioannis Panageas was supported by NSF grant CCF-2454115.  This work was also supported by the French National Research Agency (ANR) in the framework of the PEPR IA FOUNDRY project (ANR-23-PEIA-0003) and project MIS 5154714 of the National Recovery and Resilience Plan Greece 2.0 funded by the European Union under the NextGenerationEU Program.

\bibliographystyle{alpha}  
\bibliography{main}

\clearpage
\tableofcontents

\newpage
\appendix






\section{Further Related Work}

\paragraph{Other Strong Notions of Regret}

Another well-studied strong notion of regret is the internal regret, which was initially introduced in \cite{foster1999regret} in the context of calibrated forecasting. Internal regret is a notion of regret, weaker than swap and stronger than external regret, which measures the change in performance by substituting every occurrence of a given action $a$ by an alternative action $a'$.  Several algorithms minimizing internal regret have been developed, e.g., \cite{stoltz2005internal, blum2007externalMansour, cesa2003potential, sharma2024noAAAI} to name a few.
Moreover, Mohri and Yang \cite{mohri2014conditional, mohri2017online} have introduced alternative strong notions of  regret, namely {conditional swap regret} and transductive regret, which account for all possible modifications of a learner’s action sequence that depend on some fixed bounded history.

\paragraph{Further related work on adversarial combinatorial bandits.}

\cite{bubeck2012towards} proposed the use of John's ellipsoid method for optimal experiment design, achieving a regret bound of $\widetilde{\mathcal{O}}(\sqrt{mdT})$ for any finite action set, against a lower bound of $\widetilde{\mathcal{O}}(\sqrt{dT})$ \cite{audibert2014regret};
in general however, John's ellipsoid method requires time polynomial in $N$, so it is not always efficient.
An alternative \textendash\ and, in many cases, more efficient \textendash\ exploration method is via approximate barycentric spanners \cite{awerbuch2004adaptive,awerbuch2008spanners,dani2007price,bartlett2008high,2016_combexp_comband}, or the more sophisticated approximate volumetric spanners \cite{hazan2016volumetric}.

\allowdisplaybreaks

\section{Lower Bound}\label{sec: lower_bound}

In \cite{daskalakis2024lower}
a swap regret lower bound is derived for adversarial online learning in the context of extensive-form games (EFGs) with full information. The authors study EFGs from the perspective of a single player as an online learning setting and show that achieving average swap regret $\epsilon$ against an oblivious adversary in EFGs with $d$ decision nodes requires a number of rounds $T$ that is exponential either in $d$ or $\epsilon$. Moreover, as shown in \cite{daskalakis2024lower} the actions of a single player in EFGs can be written as binary vectors of dimensionality $d$ and the rewards are linear in the action vectors. Therefore, the above setting is captured by the combinatorial setting of our paper, and the above lower bound is applicable to combinatorial bandits. 
In particular, Theorem 4.1 in \cite{daskalakis2024lower} states that there exists an  action set $\mathcal{A} \subset \{0,1\}^d$, such that for any $\epsilon>0$ and $T \leq \exp(\Omega(\min\{d^{1/14}, \epsilon^{-1/6}\}))$, there exists an oblivious adversary running for $T$ rounds against whom no learner can achieve average swap regret less than $\epsilon$. Since the lower bound holds for full information feedback, it also carries over to the bandit case. 


In the following corollary we restate the result for the regime of interest, that is $T << \exp(d)$, in terms of the regret as a function of $T$.

\begin{corollary*}
For $T \leq \exp(\Omega(d^{1/14}))$ there exists a combinatorial bandit setting with action set $\mathcal{A} \subset \{0,1\}^d$ and an oblivious adversary running for $T$ rounds against whom any learner suffers swap regret at least $\Omega\left(\frac{T}{\log^{6} T}\right)$.
\end{corollary*}

Therefore, in the regime of interest (that is, $T << \exp(d)$), our swap regret upper bound is tight, in the sense that there exists no $T^p$-no-swap-regret scheme, where $p$ is a constant less than 1.

\section{Proof of Theorem \ref{lem: swap_regret_}}\label{appendix: proof_tree_swap}
We first prove Proposition \ref{property_laziness} which describes the useful property of laziness in swap regret.

\textbf{Proposition 3.1} (Swap Regret is upper bounded by External Regret for time invariant policies):
    \textit{For any policy $p\in\Delta(\mathcal{A})$ that is fixed in an interval $t=1...T$, any sequence of rewards $R_t$ and any swap function $\phi: \mathcal{A}\rightarrow \mathcal{A} $ it holds that}
    $$\sum_{t=1}^T\sum_{M \in \mathcal{A}}p(M) R_t\cdot \phi(M)\leq \max_{M \in \mathcal{A}} \left\{ \sum_{t=1}^T R_t\cdot M \right\}$$

\begin{proof}
\begin{align*}
    \sum_{t=1}^T\sum_{M \in \mathcal{A}}p(M) R_t\cdot \phi(M) &= \sum_{t=1}^T R_t\cdot \sum_{M \in \mathcal{A}}p(M) \phi(M) \\
    &=\left(\sum_{t=1}^T R_t \right) \cdot \left(\sum_{M \in \mathcal{A}}p(M) \phi(M) \right)\\
    &= \left(\sum_{t=1}^T R_t \right)\cdot V, \text{\qquad for some $V\in Conv(\mathcal{A})$}\\
    &\leq \max_{V \in Conv(\mathcal{A})} \left\{ \left(\sum_{t=1}^T R_t \right)\cdot V \right\}\\
    &\leq \max_{M \in \mathcal{A}} \left\{ \sum_{t=1}^T R_t\cdot M \right\}\\
\end{align*}
In the last step we used the fact that linear functions achieve their maximum at extreme points of the feasible region. In the case of the convex hull, the extreme points are precisely the original set of vectors.
\end{proof}

\begin{figure}[t]
\centering
\includegraphics[scale=0.07]{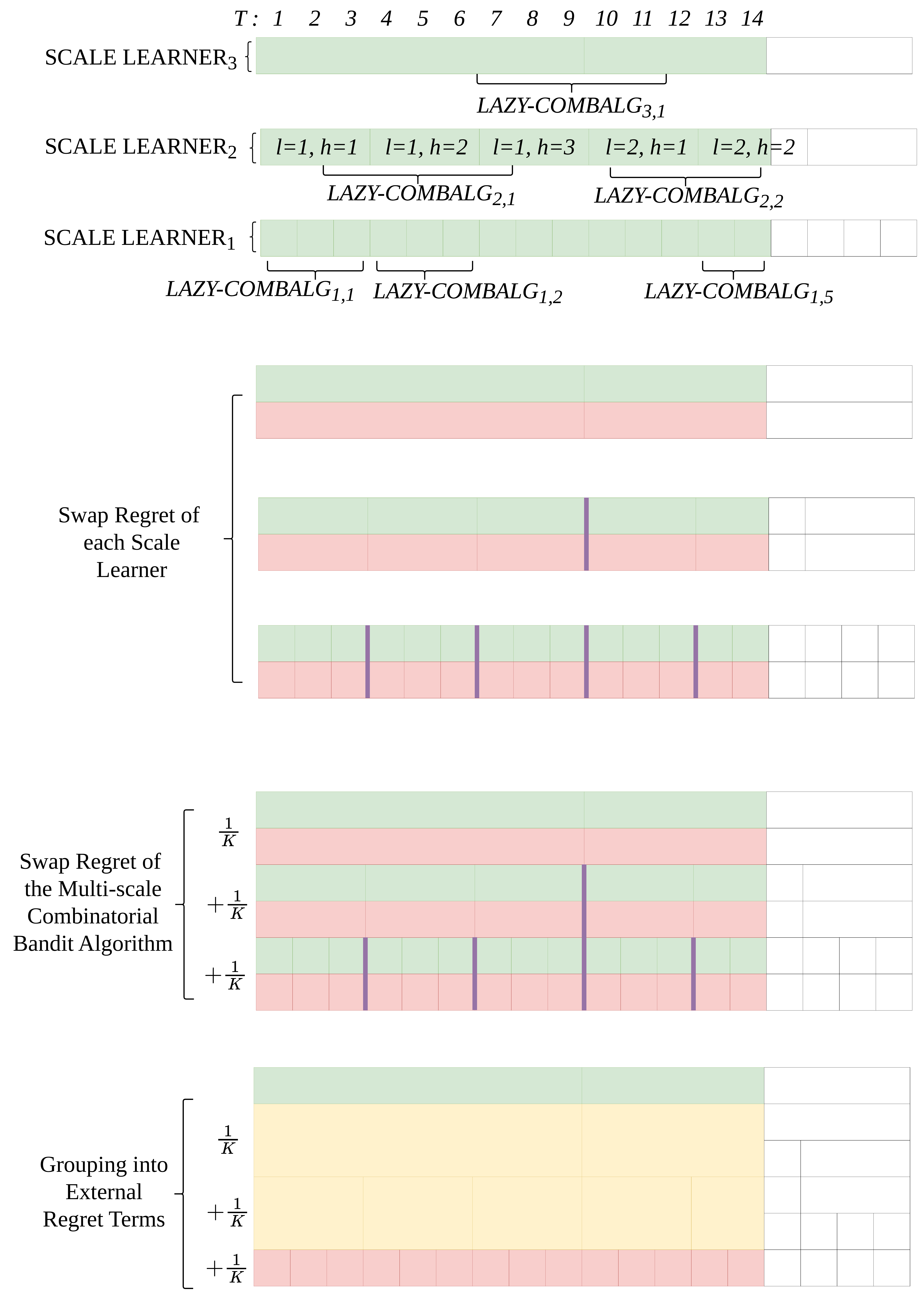}
\caption{Swap Regret Analysis for the
Multi-scale Combinatorial Bandit Algorithm for $T=14, K=3$ and $H=3$. Purple lines represent restarts of the ScaleLearners. Green terms correspond to expected cumulative rewards collected by the learners, while red terms correspond to expected cumulative rewards of a swap policy that is optimal in hindsight. In each interval, where a \textsc{Lazy-CombAlg}'s policy remains fixed, the optimal swap policy is playing a fixed pure action (see Proposition \ref{property_laziness}). Thus, red terms actually  correspond to expected cumulative rewards of fixed actions that are optimal in hindsight. Orange terms represent the external regrets of the lazy no-external regret algorithms.}\label{fig_tree_swap}
\end{figure}

\newpage

\paragraph{Proof of Theorem \ref{lem: swap_regret_}}

\begin{proof}

Due to space constraints, here we use the notation $\text{SR}$ to denote $\swap$.

\begin{align} 
\text{SR}_{T}({\phi})&=\mathbb{E}\left[\sum_{t=1}R_t \cdot (\phi(M_t) - M_t)\right]\\
&=\mathbb{E}\left[\sum_{t=1}\mathbb{E}_{\widehat{p}_t}\left[R_t \cdot (\phi(M_t) - M_t) \;\middle|\; \mathcal{F}_{t} \right] \right] \label{step:1} \\
&=\mathbb{E}\left[\frac{1}{K}\sum_{k=1}^{K}\sum_{t=1}^T\sum_{M \in \mathcal{A}}p_{k,t}(M) R_t\cdot\left(\phi(M) - M\right)\right] \label{step:2} \\
& \leq \mathbb{E}\left[\frac{1}{K}\sum_{k=1}^{K}\sum_{l=1}^{\floor*{\frac{T}{H^k}}}\sum_{h=1}^{H}\sum_{M \in \mathcal{A}}p_{k,l,h}(M) \sum_{\tau=1}^{H^{k-1}} R_{k,l,h,\tau}\cdot\left(\phi(M) - M\right)\right] +\nonumber \\
&+ \mathbb{E}\left[\frac{1}{K}\sum_{k=1}^{K}\sum_{h=1}^{H'_k}\sum_{M \in \mathcal{A}}p_{k,l,h}(M) \sum_{\tau=1}^{D_{h,k}} R_{k,l,h,\tau}\cdot\left(\phi(M) - M\right)\bigg\rvert_{l=\floor*{\frac{T}{H^k}}+1}\right]  \label{step:3} \\
& \leq \mathbb{E}\left[ \frac{1}{K}\sum_{k=1}^{K}\sum_{l=1}^{\floor*{\frac{T}{H^k}}}\sum_{h=1}^{H}\left[\max_{M \in \mathcal{A}}\left\{\sum_{\tau=1}^{H^{k-1}} R_{k,l,h,\tau}\cdot M\right\}  -\sum_{M \in \mathcal{A}}p_{k,l,h}(M) \sum_{\tau=1}^{H^{k-1}} R_{k,l,h,\tau}\cdot M \right] \right] + \nonumber \\
&+\mathbb{E}\left[ \frac{1}{K}\sum_{k=1}^{K}\sum_{h=1}^{H'_k}\left[\max_{M \in \mathcal{A}}\left\{\sum_{\tau=1}^{D_{h,k}} R_{k,l,h,\tau}\cdot M\right\}  -\sum_{M \in \mathcal{A}}p_{k,l,h}(M) \sum_{\tau=1}^{D_{h,k}} R_{k,l,h,\tau}\cdot M \right]\bigg\rvert_{l=\floor*{\frac{T}{H^k}}+1} \right]\label{step:4} \\
& = \mathbb{E}\left[ \frac{1}{K}\sum_{k=1}^{K}\sum_{l=1}^{\floor*{\frac{T}{H^k}}}\sum_{h=1}^{H}\left[\max_{M \in \mathcal{A}}\left\{X_{k,l,h}\cdot M \right\}  -\sum_{M \in \mathcal{A}}p_{k,l,h}(M) X_{k,l,h}\cdot M \right] \right] +\nonumber \\
&+\mathbb{E}\left[ \frac{1}{K}\sum_{k=1}^{K}\sum_{h=1}^{H'_k}\left[\max_{M \in \mathcal{A}}\left\{X_{k,l,h}\cdot M\right\}  -\sum_{M \in \mathcal{A}}p_{k,l,h}(M) X_{k,l,h}\cdot M \right]\bigg\rvert_{l=\floor*{\frac{T}{H^k}}+1} \right]\label{step:5} \\
& = \mathbb{E}\left[\frac{1}{K}\sum_{k=2}^{K}\sum_{l=1}^{\floor*{\frac{T}{H^k}}}\sum_{h=1}^{H}\max_{M \in \mathcal{A}}\left\{X_{k,l,h}\cdot M \right\}  + \frac{1}{K}\sum_{l=1}^{\floor*{\frac{T}{H}}}\sum_{h=1}^{H}\max_{M \in \mathcal{A}}\left\{X_{1,l,h}\cdot M \right\} \right]  +\nonumber  \\ 
&+\mathbb{E}\left[ \frac{1}{K}\sum_{k=2}^{K}\sum_{h=1}^{H'_k}\left[\max_{M \in \mathcal{A}}\left\{X_{k,l,h}\cdot M\right\} \right]\bigg\rvert_{l=\floor*{\frac{T}{H^k}}+1} \right]
+\mathbb{E}\left[ \frac{1}{K}\sum_{h=1}^{H'_1}\left[\max_{M \in \mathcal{A}}\left\{X_{1,l,h}\cdot M\right\} \right]\bigg\rvert_{l=\floor*{\frac{T}{H}}+1} \right] -\nonumber \\
&- \mathbb{E}\left[\frac{1}{K}\sum_{k=1}^{K-1}\sum_{l=1}^{\floor*{\frac{T}{H^k}}}\sum_{h=1}^{H}\sum_{M \in \mathcal{A}}p_{k,l,h}(M) X_{k,l,h}\cdot M+\frac{1}{K}\sum_{l=1}^{\floor*{\frac{T}{H^K}}}\sum_{h=1}^{H}\sum_{M \in \mathcal{A}}p_{K,l,h}(M) X_{K,l,h}\cdot M \right] -\nonumber \\
&-\mathbb{E}\left[ \frac{1}{K}\sum_{k=1}^{K-1}\sum_{h=1}^{H'_k}\left[\sum_{M \in \mathcal{A}}p_{k,l,h}(M) X_{k,l,h}\cdot M \right]\bigg\rvert_{l=\floor*{\frac{T}{H^k}}+1} \right]
-\mathbb{E}\left[ \frac{1}{K}\sum_{h=1}^{H'_K}\left[\sum_{M \in \mathcal{A}}p_{K,l,h}(M) X_{K,l,h}\cdot M \right]\bigg\rvert_{l=\floor*{\frac{T}{H^K}}+1} \right]\label{step:6} \\
& \leq \mathbb{E}\left[\frac{1}{K}\sum_{k=2}^{K}\sum_{l=1}^{\floor*{\frac{T}{H^k}}}\sum_{h=1}^{H}\max_{M \in \mathcal{A}}\left\{X_{k,l,h}\cdot M \right\} \right] 
+\mathbb{E}\left[ \frac{1}{K}\sum_{k=2}^{K}\sum_{h=1}^{H'_k}\left[\max_{M \in \mathcal{A}}\left\{X_{k,l,h}\cdot M\right\} \right]\bigg\rvert_{l=\floor*{\frac{T}{H^k}}+1} \right]+\frac{T}{K} -\nonumber \\
&- \mathbb{E}\left[\frac{1}{K}\sum_{k=1}^{K-1}\sum_{l=1}^{\floor*{\frac{T}{H^k}}}\sum_{h=1}^{H}\sum_{M \in \mathcal{A}}p_{k,l,h}(M) X_{k,l,h}\cdot M\right] -\mathbb{E}\left[ \frac{1}{K}\sum_{k=1}^{K-1}\sum_{h=1}^{H'_k}\left[\sum_{M \in \mathcal{A}}p_{k,l,h}(M) X_{k,l,h}\cdot M \right]\bigg\rvert_{l=\floor*{\frac{T}{H^k}}+1} \right]\label{step:7} \\
&= \frac{1}{K}\sum_{k=1}^{K-1}\sum_{l=1}^{\ceil*{\frac{T}{H^{k}}}}\reg(\textsc{Lazy-CombAlg}_{k,l})+\frac{T}{K} \label{basic_step}
\end{align}

In \ref{step:1} we use the tower property, and in \ref{step:2} we use the definition of the conditional expectation and the definition of the master's distribution $\widehat{p}_t$. When we write $R_{k,l,h,\tau}$ we mean $R_t$ at $t=(l-1)H^k+(h-1)H^{k-1}+\tau$.

One crucial part is step \ref{step:3}, where for each \textsc{ScaleLearner}$_k$ we split $T$ into two parts; the first part contains the segments in which the \textsc{Lazy-CombAlg} instance fully runs the horizon, and one final part in which the \textsc{Lazy-CombAlg} instance may not fully run the horizon, due to the fact that $H^{K}$ may not perfectly divide $T$. In particular, we consider the following: (a) in the first part, for each \textsc{ScaleLearner}$_k$ we have $\floor*{\frac{T}{H^k}}$ segments of size $H \cdot H^{k-1}$ days, where the \textsc{Lazy-CombAlg} instance fully runs a horizon of $H$ meta-days, (b) in the second part, for each learner, there remain $H'_k$ meta-days, where $0 \leq H'_k\leq H$ (that is at most $H^{k}$ days), where the \textsc{Lazy-CombAlg} instance may terminate its run before a horizon of $H$ meta-days is completed. Moreover, in this second part the last meta day of each learner may last less than $H^{k}$ days and we denote this duration by $D_{h,k}$. 

In step \ref{step:4} we use a key property of laziness that allows us to upper bound the performance of the reference swap function $\phi$ by the performance of the best fixed action in an interval where the learner has fixed policy, as we show n Proposition \ref{property_laziness}.

In step \ref{step:5} we define  $$X_{k,l,h}:= \left\{
\begin{array}{ll}
      \sum_{\tau=1}^{H^{k-1}} R_{k,l,h,\tau} & 1\leq k\leq K, \quad  1\leq l\leq\floor*{\frac{T}{H^k}}, \quad 1\leq h \leq H\\
      \sum_{\tau=1}^{D_{h,k}} R_{k,l,h,\tau} & 1\leq k\leq K, \quad  l=\floor*{\frac{T}{H^k}}+1, \quad 1\leq h \leq H'_k\\
\end{array} 
\right. $$

In steps \ref{step:6}, \ref{step:7} and \ref{basic_step} we group the terms corresponding to the cumulative reward of the optimal fixed action in each meta-day and the terms corresponding to the expected cumulative reward of the lazy learners so that the external regrets of the lazy learners are formed. For the swap regret upper bound we ignore the expected cumulative reward of the layer $k=K$, since its contribution only decreases the upper bound. Moreover, we isolate the cumulative reward of the optimal fixed actions in the meta-days at layer $k=1$ (these meta-days are actually days since they have duration $H^{k-1}=1$). These terms can not be paired into any external regret term and they increase the swap regret upper bound, but we can bound their total contribution by $\frac{T}{K}$, which is vanishing on $K$. This analysis can be visualized in Figure \ref{fig_tree_swap}.



\end{proof}

\medskip


\section{Anytime Guarantees of the Multi-Scale Algorithm}

\begin{lemma}[Lemma \ref{lem:anytime_} restated]
Assuming that Algorithm \ref{alg:multiscale_algo} achieves swap regret
$$\swap_{T}({\phi}) \leq \mathcal{O}\left(\frac{T\log(d\log T)}{\log T} \right),$$ 
Then, Algorithm \ref{alg:multiscale_algo} achieves the above swap regret with anytime guarantees, if it is implemented with the Doubling Trick.
\end{lemma}

\begin{proof}
    
We apply the Doubling Trick \cite{auer2002adaptive, besson2018doublingTrick} to obtain anytime guarantees for Algorithm \ref{alg:multiscale_algo}. For convenience, the analysis below ignores the terms that do not depend on $T$. 

First, we split the sum of the regrets (induced by the Doubling Trick) into two parts:
\[
S = \sum_{i=1}^{\log T} \frac{2^i \log(di)}{i} \leq \log(d\log T) \sum_{i=1}^{\log T} \frac{2^i}{i} = \log(d\log T) \underbrace{\sum_{i=1}^{\lfloor \log T / 2 \rfloor} \frac{2^i}{i}}_{\text{Part 1}} + \log(d\log T) \underbrace{\sum_{i=\lfloor \log T / 2 \rfloor + 1}^{\log T} \frac{2^i}{i}}_{\text{Part 2}}.
\]

$\bullet \quad $ For \( i \leq \log T / 2 \), note that \( 2^i \leq \sqrt{T} \), since \( i \leq \log T / 2 \) implies \( 2^i = \sqrt{2^{\log T}} = \sqrt{T} \). Therefore:
\[
\sum_{i=1}^{\lfloor \log T / 2 \rfloor} \frac{2^i}{i} \leq \sqrt{T} \cdot \sum_{i=1}^{\lfloor \log T / 2 \rfloor} \frac{1}{i}.
\]

The summation \( \sum_{i=1}^n \frac{1}{i} \) is the \( n \)-th harmonic number, \( H_n \), which satisfies:
\[
H_n \leq \ln n + 1.
\]

Substituting \( n = \lfloor \log T / 2 \rfloor \), we get:
\[
\sum_{i=1}^{\lfloor \log T / 2 \rfloor} \frac{2^i}{i} \leq \sqrt{T} \cdot \left( \ln (\log T / 2) + 1 \right).
\]

Thus, Part 1 is:
\[
\mathcal{O}\left(\sqrt{T} \cdot \log \log T\right),
\]
which is asymptotically smaller than \( \mathcal{O}\left(\frac{T}{\log T}\right) \).

\medskip

$\bullet \quad $ For \( i > \log T / 2 \), note that \( i \geq \log T / 2 \) implies \( \frac{1}{i} \leq \frac{2}{\log T} \). Therefore:
\[
\sum_{i=\lfloor \log T / 2 \rfloor + 1}^{\log T} \frac{2^i}{i} \leq \frac{2}{\log T} \cdot \sum_{i=\lfloor \log T / 2 \rfloor + 1}^{\log T} 2^i.
\]

The summation \( \sum_{i=\lfloor \log T / 2 \rfloor + 1}^{\log T} 2^i \) is a geometric sequence, which simplifies to:
\[
\sum_{i=\lfloor \log T / 2 \rfloor + 1}^{\log T} 2^i = 2^{\log T + 1} - 2^{\lfloor \log T / 2 \rfloor + 1}.
\]

Using \( 2^{\log T} = T \), this becomes:
\[
\sum_{i=\lfloor \log T / 2 \rfloor + 1}^{\log T} 2^i = 2T - 2 \cdot \sqrt{T} \leq 2T.
\]

Substituting this back, we have:
\[
\sum_{i=\lfloor \log T / 2 \rfloor + 1}^{\log T} \frac{2^i}{i} \leq \frac{2}{\log T} \cdot 2T = \frac{4T}{\log T}.
\]

Combining the above, we have that the total summation is the sum of Part 1 and Part 2. Since Part 1 is \( \mathcal{O}\left(\sqrt{T} \cdot \log \log T\right) \) and Part 2 is \( \mathcal{O}\left(\frac{T}{\log T}\right) \), the dominant term is from Part 2. Therefore:
\[
S = \sum_{i=1}^{\log T} \frac{2^i}{i} = \mathcal{O}\left(\frac{T}{\log T}\right).
\]

\end{proof}

\newpage

\section{Proof of Lemma \ref{lem:key_}}

\begin{lemma}[Lemma \ref{lem:key_} restated]\label{lem:key}
    For all $k \in \{1,...,K-1\}$, $l \in \left[\ceil*{\frac{T}{H^{k}}}\right]$ and $h \in [H]$, it holds that,

    $$ \mathbb{E}_{k,l,h}\left[ M_{k,l,h,\tau}^T \Sigma_{k,l,h,\tau}^{+2} M_{k,l,h,\tau} \right] \leq d \lambda^{-1}_{\text{min}}(\Sigma_{k,l,h,\tau}), $$

    where $\lambda^{-1}_{\text{min}}(\Sigma_{k,l,h,\tau})$ is the minimum eigenvalue of the co-occurrence matrix $\Sigma_{k,l,h,\tau}$ under the law of $\hat{p}_{k,l,h,\tau}$.
\end{lemma}

\begin{proof}

Let $\lambda_i(\Sigma_{k,l,h,\tau})$ be the eigenvalue of $\Sigma_{k,l,h,\tau}$ corresponding to the eigenvector $u_i(\Sigma_{k,l,h,\tau})$. For brevity, we abuse notation and instead simply use $\lambda_{\tau,i}$ and $u_{\tau,i}$. Moreover, let $\lambda_{\tau,1}, ..., \lambda_{\tau,r}$ be the $r$ nonzero eigenvalues with corresponding eigenvectors $u_{\tau,1}, ..., u_{\tau,r}$. Based on the above, we have the following:

\begin{align}
\mathbb{E}_{k,l,h}\left[ M_{k,l,h,\tau}^T \Sigma_{k,l,h,\tau}^{+2} M_{k,l,h,\tau} \right] 
& = \mathbb{E}_{k,l,h}\left[ \mathbb{E}_{k,l,h,\tau}\left[ M_{k,l,h,\tau}^T \Sigma_{k,l,h,\tau}^{+2} M_{k,l,h,\tau} \right] \right] \label{variance_8} \\
& = \mathbb{E}_{k,l,h}\left[ \sum_{M\in\mathcal{A}} \widehat{p}_{k,l,h,\tau}(M) M^T \Sigma_{k,l,h,\tau}^{+2} M \right] \\
& \leq \mathbb{E}_{k,l,h}\left[ \sum_{M\in\mathcal{A}} \widehat{p}_{k,l,h,\tau}(M) \sum_{i \in [r]} \lambda_{\tau, i}^{-2} \cdot (u_{\tau,i} \cdot M)^2 \right] \label{variance_9} \\
& = \mathbb{E}_{k,l,h}\left[ \sum_{i \in [r]} \lambda_{\tau, i}^{-2} \sum_{M\in\mathcal{A}} \widehat{p}_{k,l,h,\tau}(M) (u_{\tau,i} \cdot M)^2 \right] \\
& = \mathbb{E}_{k,l,h}\left[ \sum_{i \in [r]} \lambda_{\tau, i}^{-1} \right] \label{variance_10} \\ 
& \leq d \lambda^{-1}_{\text{min}}(\Sigma_{k,l,h,\tau}) \label{variance_9.5}
\end{align}

where in \ref{variance_8} we used the tower property, in \ref{variance_9} we used: (a) the inequality $x^T A x \leq \sum_{i \in [r]} \lambda_i (u_i \cdot x)$ for any vertex $x \in \mathbb{R}^d$ and any symmetric positive semi-definite matrix $A \in \mathbb{R}^{d\times d}$, and (b) the fact that the eigenvalues of the squared inverse matrix $A^{+2}$ equal the squared inverse of the eigenvalues of $A$.

\begin{align}
    \lambda_{\tau,i} = u_i^T \Sigma_{k,l,h,\tau} u_i = \sum_{M \in \mathcal{A}} \widehat{p}_{k,l,h,\tau}(M) (u_i \cdot M)^2. \label{eq:eigenvalue_prop}
\end{align}

Moreover, in \ref{variance_9.5} we used the fact that $r \leq d$.  
\end{proof}

\newpage

\section{\textsc{Swap-ComBCP}}

\subsection{Barycentric Spanners}\label{sec: barycentric}

This section introduces the notion of barycentric spanners \cite{awerbuch2004adaptive}. In the Algorithm \ref{alg:lazy_combcp}, we leverage barycentric spanners to ensure adequate exploration of each coordinate $i \in [d]$ that allows us to guarantee low variance of the reward estimators.

\begin{definition}[$C$-approximate barycentric spanner]
A subset of vectors $S = \left\{b_1, \ldots, b_d\right\} \subseteq \mathcal{A}$ is said to be $C$-approximate barycentric spanner of $\mathcal{A}$, with $C > 1$, if, for all $M \in \mathcal{A}$, there exists $a \in \mathbb{R}^d$ such that
$$ M= Ba \quad \text { and } \quad \|a\|_{\infty} \leq C.
$$
where $B$ is a matrix whose columns are the barycentric spanners $\left\{b_1, \ldots, b_d\right\}$.
\end{definition} 

The following proposition ensures that, if specific conditions hold, there exists an efficient algorithm for computing a $C$-approximate barycentric spanner.

\begin{proposition}[\cite{awerbuch2008spanners}, Proposition 2.5]\label{prop:barycentric_algo}
Suppose $\mathcal{A} \subseteq \mathbb{R}^d$ is a compact set not contained in any proper linear subspace. Given an oracle for optimizing linear functions over $\mathcal{A}$, for any $C>1$ there exists an algorithm that computes a $C$-approximate barycentric spanner for $\mathcal{A}$ in polynomial time, using $\widetilde{\mathcal{O}}\left(d^2\right)$ calls to the optimization oracle.
\end{proposition}




In order to control the minimum eigenvalue of the co-occurrence matrix under the law of $\mu$ (uniform distribution over the barycentric spanner), we will make use of the following lemma. 

\begin{lemma}[Lemma \ref{lem: 2016_lem_} restated]\label{lem: 2016_lem}
    Let $\mu$ be the uniform distribution over the $C$-approximate barycentric spanner, $S$, of $\mathcal{A} \subset \mathbb{R}^d$, Also, let $\lambda_{\text{min}}(\mu)$ be the minimum eigenvalue of the co-occurrence matrix under the law of $\mu$. Then, it holds that,
    $$ \lambda_{\text{min}}(\mu) \geq \frac{1}{C^2 d^3}.
    $$
\end{lemma}

\begin{proof}

Let $v_i$, for $i \in [d]$, be the spanner vectors of $S$, and let $B$ be the matrix with $v_i$ as columns. Using the barycentric spanner property, we know that the coefficients $\lambda_{M, i}$ for all $M \in A$ satisfy the following:
$$
M=\sum_i \lambda_{M, i} v_i = B\boldsymbol{\lambda}, \quad\left|\lambda_{M, i}\right| \leq C
$$

where $\boldsymbol{\lambda}$ being the vector containing the coefficients. For all $M \in span(\mathcal{A})$, we will show that $MM^T = B \lambda \lambda^T B^T \preceq nC^2 BB^T$. We have that,

\begin{align}
x^T \left( B \boldsymbol{\lambda} \boldsymbol{\lambda}^T B^T - d C^2 BB^T \right) x 
&= x^T B \left( \boldsymbol{\lambda} \boldsymbol{\lambda}^T - d C^2 I \right) B^T x \nonumber \\
&= y^T \left( \boldsymbol{\lambda} \boldsymbol{\lambda}^T - d C^2 I \right) y, \quad \text{where } y = B^T x \nonumber \\
&= y^T \boldsymbol{\lambda} \boldsymbol{\lambda}^T y - d C^2 y^T y \nonumber \\
&= (y \cdot \boldsymbol{\lambda} \boldsymbol{\lambda})^2 - d C^2 \|y\|^2 \nonumber \\
&\leq \| \boldsymbol{\lambda}\|^2 \|y\|^2 - d C^2 \|y\|^2 \nonumber \\
&= \left( \|\boldsymbol{\lambda}\|^2 - d C^2 \right) \|y\|^2 \nonumber \\
& \leq 0, \quad \text{because } \|\boldsymbol{\lambda}\|^2 \leq d C^2. \nonumber
\end{align}

Therefore, since $MM^T$ and $BB^T$ are symmetric matrices, we have that, 

\begin{align}    
MM^T &\preceq dC^2BB^T \\
&= d^2C^2 \widetilde{B}
\end{align}

where we denote by $\widetilde{B}$ the co-occurrence matrix under the law of $\mu$.

Now, let $\mu'$ be any distribution on $\mathcal{A}$. Then, we have that $\mathbb{E}_{M \sim \mu'}\left[M M^{T}\right] \preceq C^2 d^2 \widetilde{B}$, which implies that, 

$$ \lambda_{\text{min}}(\mu) \geq \frac{\lambda_{\text{min}}(\mu')}{C^2 d^2}.$$

However, from \cite{cesa2012combinatorial, bubeck2012towards}, we know that there exists a distribution with a minimum eigenvalue of the induced co-occurrence matrix being at least $1/d$. This distribution is the John's exploration distribution, which always exists for the action sets of combinatorial bandits and provides a minimum eigenvalue of the co-occurrence matrix being at least $1/d$ (see \cite{bubeck2012towards}, last sentence of the proof of Theorem 4, and in \cite{cesa2003potential} page 9: "$\lambda_{min} = \Omega(1/d)$ is achievable for all classes $S$").  

\end{proof}

\subsection{Analysis of \textsc{Swap-ComBCP}}

\begin{lemma}[Lemma \ref{bcp: properties} and Lemma \ref{lem:combcp_bounded_variance} combined]\label{bcp: properties}
For all $k \in \{1,...,K-1\}$, $l \in \left[\ceil*{\frac{T}{H^{k}}}\right]$ and $h \in [H]$, $\textsc{Lazy-ComBCP}$ satisfies the following properties:

\begin{enumerate}
    \item  (Unbiasedness): $ \mathbb{E}[mq_{k,l,h} \cdot\widetilde{X}_{k,l,h}] = \mathbb{E}[mq_{k,l,h} \cdot{X}_{k,l,h}] $
    \item (Boundedness): $ \|\widetilde{R}_{k,l,h,\tau}\|_{2} \leq \frac{C^2 d^3 \sqrt{m}}{\gamma} $
    \item (Bounded variance): $ \mathbb{E}[q_{k,l,h} \cdot \widetilde{X}_{k,l,h}^2] \leq \frac{H^{2k-2} d^4 C^2}{\gamma} $
\end{enumerate}

\end{lemma}

\begin{proof}

$\bullet \quad \quad$ 1.

\medskip

For each $(k,l,h)$, the following hold:

\begin{align}
    \sum_{\tau=1}^{H^{k-1}} \mathbb{E}\left[\widetilde{R}_{k,l,h,\tau} \;\middle|\; \mathcal{F}_{k,l,h,\tau}\right]
    &= \sum_{\tau=1}^{H^{k-1}} \mathbb{E}\left[r_{k,l,h,\tau}\Sigma_{k,l,h,\tau}^{-1}M_{k,l,h,\tau} \;\middle|\; \mathcal{F}_{k,l,h,\tau}\right] \\
    &= \sum_{\tau=1}^{H^{k-1}} \mathbb{E}\left[\Sigma_{k,l,h,\tau}^{-1}M_{k,l,h,\tau} M_{k,l,h,\tau} \cdot {R}_{k,l,h,\tau}\;\middle|\; \mathcal{F}_{k,l,h,\tau}\right] \\
    &= \sum_{\tau=1}^{H^{k-1}} \left( \Sigma_{k,l,h,\tau}^{-1} \Sigma_{k,l,h,\tau} R_{k,l,h,\tau}\right) \\
    &= \sum_{\tau=1}^{H^{k-1}} {R}_{k,l,h,\tau}
\end{align}

Using the above, we have 

\begin{align}
    mq_{k,l,h} \cdot \sum_{\tau=1}^{H^{k-1}} \mathbb{E}\left[ \widetilde{R}_{k,l,h,\tau} \;\middle|\; \mathcal{F}_{k,l,h,\tau}\right] &= mq_{k,l,h} \cdot \sum_{\tau=1}^{H^{k-1}} {R}_{k,l,h,\tau} \\
    &= mq_{k,l,h} \cdot X_{k,l,h}\label{unbiased_mqX_1}
\end{align}

\medskip

\noindent
Taking expectations and using the tower property in \cref{unbiased_mqX_1} we get

\begin{align}
    \mathbb{E}[mq_{k,l,h} \cdot\widetilde{X}_{k,l,h}] = \mathbb{E}[mq_{k,l,h} \cdot{X}_{k,l,h}]\label{unbiased_mqX_2}
\end{align}

$\bullet \quad \quad$ 2.

\medskip

For all $k \in \{2,...,K+1\}$, $l \in \left[\ceil*{\frac{T}{H^{k}}}\right]$ and $h \in [H]$, we have:

\begin{align} \|\widetilde{R}_{k,l,h,\tau}\|_{\infty} 
& \leq \|\widetilde{R}_{k,l,h,\tau}\|_2 \\ 
& =\left\|r_{k,l,h,\tau} \Sigma^{-1}_{k,l,h} M_{k,l,h,\tau}\right\|_2 \\ 
& \leq  \left\|\Sigma_{k,l,h,\tau}^{-1} M_{k,l,h,\tau}\right\|_2 \\
& =  \sqrt{M_{k,l,h,\tau}^{\top} \Sigma_{k,l,h,\tau}^{-1} \Sigma_{k,l,h}^{-1} M_{k,l,h,\tau}} \\ 
& \leq \|M_{k,l,h,\tau}\|_2 \sqrt{\lambda_{\max }\left(\Sigma_{k,l,h,\tau}^{-1} \Sigma_{k,l,h,\tau}^{-1}\right)} \\ 
& =\sqrt{m} \sqrt{\lambda_{\max }\left(\Sigma_{k,l,h,\tau}^{-1} \Sigma_{k,l,h,\tau}^{-1}\right)} \\ 
& =\sqrt{m} \lambda_{\max }\left(\Sigma_{k,l,h,\tau}^{-1}\right) \\ 
& =\frac{\sqrt{m}}{\lambda_{\min }\left(\Sigma_{k,l,h,\tau}\right)} \\
& \leq \frac{\sqrt{m}}{\gamma \lambda_{\min }\left(\widetilde{B}\right)} \label{weil} \\
& \leq \frac{C^2 d^3 \sqrt{m}}{\gamma} \label{uniform_1}
\end{align}

where in \ref{weil} we defined $\widetilde{B}$ to be the co-occurrence matrix under the law of the uniform distribution over the barycentric spanner, and also we made use of the Weyl's inequality. We also define $B$ to be a matrix whose columns are the vectors of the approximate barycentric spanner $S$. In \ref{uniform_1} we used Lemma \ref{lem: 2016_lem}. 

\medskip

We also have that,

\begin{align} 
\sup _{{M} \in \mathcal{A}} {M}^T \widetilde{B}^{-1} {M} & \leq \sup _{\alpha \in[-C,C]^d} {\alpha}^T B^T \widetilde{B}^{-1} B {\alpha} \\ 
& \leq \sup _{\|\alpha\|=\|\beta\|=C\sqrt{d}} {\alpha}^T B^T \widetilde{B}^{-1} B {\alpha} \\
& \leq C^2d \cdot \|B^T \widetilde{B}^{-1} B\|_2 \\ 
&= C^2d \cdot \lambda_{\max }\left(B^T \widetilde{B}^{-1} B\right) \\ 
& = C^2d \cdot \frac{1}{\lambda_{\text{min}}\left(B^{-1} \widetilde{B} B^{-T}\right)}  \\ 
& =C^2\frac{d}{\lambda_{\min }\left(B^{-1}\left(\frac{1}{d} B B^T\right) B^{-T}\right)} \\ 
& \leq C^2 d^2
\end{align}





\bigskip

$\bullet \quad \quad$ 3. 

\medskip

For all $k \in \{1,...,K-1\}$, $l \in \left[\ceil*{\frac{T}{H^{k}}}\right]$ and $h \in [H]$, we have:

\begin{align}
     \mathbb{E}_{k,l,h}\left[q_{k,l,h} \cdot \widetilde{X}_{k,l,h}^2\right]
     & =  \frac{1}{m} \mathbb{E}_{k,l,h}\left[mq_{k,l,h} \cdot \widetilde{X}_{k,l,h}^2\right] \\
     & = \frac{1}{m} \mathbb{E}_{k,l,h}\left[\sum_{M \in \mathcal{A}} \widetilde{p}_{k,l,h}(M) M \cdot \widetilde{X}_{k,l,h}^2\right] \label{bcp: variance_1} \\
     & = \frac{1}{m} \mathbb{E}_{k,l,h}\left[\widehat{M}_{k,l,h} \cdot \widetilde{X}_{k,l,h}^2\right] \label{bcp: variance_2} \\
     & = \frac{1}{m} \mathbb{E}_{k,l,h}\left[\widehat{M}_{k,l,h} \cdot \left(\sum_{\tau=1}^{H^{k-1}} \widetilde{R}_{k,l,h,\tau}\right)^2\right] \\
     & \leq \frac{H^{k-1}}{m} \mathbb{E}_{k,l,h}\left[\widehat{M}_{k,l,h} \cdot \sum_{\tau=1}^{H^{k-1}} \widetilde{R}_{k,l,h,\tau}^2\right] \label{bcp: variance_3}  \\
    & = \frac{H^{k-1}}{m}  \sum_{\tau=1}^{H^{k-1}} \mathbb{E}_{k,l,h}\left[\widehat{M}^2_{k,l,h} \cdot  \widetilde{R}_{k,l,h,\tau}^2 \right]\label{bcp: variance_5}  \\
    & \leq \frac{H^{k-1}}{m} \sum_{\tau=1}^{H^{k-1}} \mathbb{E}_{k,l,h}\left[M_{k,l,h,\tau}^T \Sigma_{k,l,h,\tau}^{-1} \widehat{M}_{k,l,h} \widehat{M}_{k,l,h}^T  \Sigma_{k,l,h,\tau}^{-1} M_{k,l,h,\tau} \right] \label{bcp: variance_33} \\
    & \leq \frac{H^{k-1}}{m}\sum_{\tau=1}^{H^{k-1}} \mathbb{E}_{k,l,h}\left[\lambda_{\text{max}}(\widehat{M}_{k,l,h} \widehat{M}_{k,l,h}^T) M_{k,l,h,\tau}^T \Sigma_{k,l,h,\tau}^{-2} M_{k,l,h,\tau} \right] \label{bcp: variance_6} \\
    & \leq \frac{H^{k-1}}{m}\cdot \max_{M \in \mathcal{A}} \lambda_{\text{max}}(MM^T) \cdot \sum_{\tau=1}^{H^{k-1}} \mathbb{E}_{k,l,h}\left[ M_{k,l,h,\tau}^T \Sigma_{k,l,h,\tau}^{-2} M_{k,l,h,\tau} \right] \\
    & \leq H^{k-1}  \sum_{\tau=1}^{H^{k-1}} \mathbb{E}_{k,l,h}\left[ M_{k,l,h,\tau}^T \Sigma_{k,l,h,\tau}^{-2} M_{k,l,h,\tau} \right] \label{bcp: variance_7} 
\end{align}

\medskip

where in \ref{bcp: variance_1} we used the Carathéodory decomposition step of the algorithm, in \ref{bcp: variance_2} we defined $\widehat{M}_{k,l,h}$ to be a random pure action under the law of $\widetilde{p}_{k,l,h}$, independent of $M_{k,l,h,\tau}$ which is under the law of $\widehat{p}_{k,l,h}$, in \ref{bcp: variance_3} we used the inequality $(a_1 + ... + a_N)^2 \leq N\sum_{i=1}^{N}a_i^2$, in \ref{bcp: variance_5} we used the fact $\widehat{M}_{k,l,h} = \widehat{M}_{k,l,h}^2$ (squares are in an element-wise manner), in \ref{bcp: variance_33} we used the the fact that the bandit feedback reward $r_{k,l,h,\tau} \leq 1$, and also we expand the inner product, in \ref{bcp: variance_6} we used the inequality $x^TABAx^T \leq \lambda_{\text{max}}(B)x^TA^2x$ for any vector $x \in \mathbb{R}^d$ and symmetric matrices $A,B \in \mathbb{R}^{d\times d}$, in \ref{bcp: variance_7} we used the fact that $\max_{M \in \mathcal{A}} \lambda_{\text{max}}(MM^T) \leq m$.

\medskip

The next thing to do is to bound each summand of the above sum. Before doing so, we note that each of these summands is an expectation under the law of the master's distribution $\widehat{p}_{k,l,h,\tau}$.  Using Lemma \ref{lem:key}, we have the following:

\begin{align}
\mathbb{E}_{k,l,h}\left[ M_{k,l,h,\tau}^T \Sigma_{k,l,h,\tau}^{-2} M_{k,l,h,\tau} \right]  
& \leq d \lambda^{-1}_{\text{min}}(\Sigma_{k,l,h,\tau}) \\
& \leq \frac{d}{\gamma \lambda_{\text{min}}(\widetilde{B})} \label{bcp: variance_11} \\
& \leq \frac{C^2 d^4}{\gamma} \label{bcp: variance_12}
\end{align}

where in \ref{bcp: variance_11} we used Weyl's inequality applying to $\lambda_{\text{min}}(\Sigma_{k,l,h,\tau}) \geq \gamma \lambda_{\text{min}}(\widetilde{B})$, and in \ref{bcp: variance_12} we used lemma \ref{lem: 2016_lem}.   

\medskip

Putting the above all together, we have the following:

\begin{align}
\mathbb{E}_{k,l,h}\left[q_{k,l,h} \cdot \widetilde{X}_{k,l,h}^2\right] 
& \leq \frac{H^{2k-2} d^4 C^2}{\gamma}
\end{align}

Finally, taking expectations in both sides we get the desired result:

\begin{align}
\mathbb{E}\left[q_{k,l,h} \cdot \widetilde{X}_{k,l,h}^2\right] 
& \leq \frac{H^{2k-2} d^4 C^2}{\gamma}
\end{align}

\end{proof}

Let $p_{k,l}^*$ be an optimal (deterministic) policy of the adversary, with $p_{k,l}^*(M_{k,l}^*) = 1$. 

\begin{proposition}[Proposition \ref{prop: opt_} restated]\label{prop: opt}
    Let $q_{k,l}^* = \frac{M^*_{k,l}}{m}$. It holds that 
    $$ m q_{k,l}^* =  \sum_{M \in \mathcal{A}} p^*_{k,l}(M) M. $$
\end{proposition}

\begin{proof}
Using the fact that $mq_{k,l}^* \in co(\mathcal{M})$, from the Carathéodory theorem there exists a distribution (policy), $\widehat{p}_{k,l}$, over $\mathcal{M}$ such that 

\begin{align}
    & \; \; \; \; \; \; \; \; \;  \sum_{M \in \mathcal{A}} \widehat{p}_{k,l}(M) M = m q_{k,l}^* \\
    &\implies \sum_{M \in \mathcal{A}} \widehat{p}_{k,l}(M) M = M_{k,l}^* \\
    &\implies \widehat{p}_{k,l} = p_{k,l}^* \label{eq_claim}
\end{align}

\end{proof}

Requiring $\eta \leq \frac{\gamma}{H^{k-1}C^2d^3\sqrt{m}}$, by Lemma \ref{bcp: properties}, we have that $\eta \|\widetilde{X}_{k,l,h}\|_{\infty} \leq 1$. This condition will be shown to be satisfied later in the analysis. We have the following lemma.

\medskip

\begin{lemma}[Lemma \ref{lem: omd_combcp_} restated]\label{lem: omd_combcp}
    If $\eta \|\widetilde{X}_{k,l,h}\|_{\infty} \leq 1$, it holds that

    \begin{align}
    \sum_{h=1}^H q_{k,l}^* \cdot \widetilde{X}_{k,l,h} - \sum_{h=1}^H q_{k,l,h} \cdot \widetilde{X}_{k,l,h} 
    &\leq \eta \sum_{h=1}^H q_{k,l,h} \cdot \widetilde{X}_{k,l,h}^2 + \frac{\text{KL}(q_{k,l}^*, q_{k,l,1})}{\eta} \nonumber \\ 
    &\leq \eta \sum_{h=1}^H q_{k,l,h} \cdot \widetilde{X}_{k,l,h}^2 + \frac{\log d}{\eta}, \nonumber 
    \end{align} 

    where $\widetilde{X}_t^2$ is the vector that is the coordinate-wise square of $\widetilde{X}_t$.
    
\end{lemma}

\begin{proof}

The proof directly follows the proof of Lemma A1 \cite{2016_combexp_comband}, and also by using the fact that

\begin{align}
\text{KL}(q_{k,l}^*, q_{k,l,1}) &= \sum_{i=1}^{d} q_{k,l}^*(i) \log \frac{q_{k,l}^*(i)}{q_{k,l,1}(i)} \\
& = - \frac{1}{m} \sum_{i: M_{k,l}^*(i)=1} \log \frac{m}{d} \\
& = \log d - \log m \label{eq: m} \\ 
& \leq \log d. 
\end{align}

where in \ref{eq: m} we used the fact that $\|M\|_1 = m$.

\end{proof}

\begin{theorem}[Theorem \ref{thm:combcp_no_regret_} restated]
The external regret of $\textsc{Lazy-ComBCP}_{k,l}$ is at most $3 H^{k-1} H^{2/3} d^3 m^{3/2} \log d$.  
\end{theorem}

\begin{proof}
    
To upper bound $\reg$ of the learner ${k,l}$, we have:

\begin{align}
\reg
&= \max_{M \in \mathcal{A}}\left\{\sum_{h=1}^{H}X_{k,l,h}\cdot M \right\} - \sum_{h=1}^{H}\mathbb{E}\left[\sum_{M \in \mathcal{A}}p_{k,l,h}(M) X_{k,l,h}\cdot M \right] \\ 
&= \sum_{h=1}^{H}X_{k,l,h}\cdot M_{k,l}^* - \sum_{h=1}^{H}\mathbb{E}\left[\sum_{M \in \mathcal{A}} \left[ (1-\gamma) \widetilde{p}_{k,l,h}(M) + \gamma \mu(M) \right] X_{k,l,h}\cdot M \right] \\ 
&= \sum_{h=1}^{H} m q_{k,l}^* \cdot X_{k,l,h} - \sum_{h=1}^{H}\mathbb{E}\left[(1-\gamma) mq_{k,l,h} \cdot X_{k,l,h} + \gamma \sum_{M \in \mathcal{A}} \mu(M) X_{k,l,h}\cdot M \right] \label{eq: apply_prop_opt} \\ 
&\leq \sum_{h=1}^{H} m q_{k,l}^* \cdot X_{k,l,h} - \sum_{h=1}^{H}\mathbb{E}\left[(1-\gamma) mq_{k,l,h} \cdot X_{k,l,h} \right] \\ 
&= \sum_{h=1}^{H} \mathbb{E}\left[m q_{k,l}^* \cdot X_{k,l,h} - mq_{k,l,h} \cdot X_{k,l,h} \right] + \gamma \sum_{h=1}^{H} \mathbb{E}\left[ mq_{k,l,h} \cdot X_{k,l,h} \right] \\ 
&\leq \sum_{h=1}^{H} \mathbb{E}\left[m q_{k,l}^* \cdot X_{k,l,h} - mq_{k,l,h} \cdot X_{k,l,h} \right] + \gamma H^{k} \label{ineq: gamma_bound}\\ 
&= \sum_{h=1}^{H} \mathbb{E}\left[m q_{k,l}^* \cdot \widetilde{X}_{k,l,h} - mq_{k,l,h} \cdot \widetilde{X}_{k,l,h} \right] + \gamma H^{k} \label{eq: unbiased_to_regret}\\ 
&\leq m \eta \sum_{h=1}^H \mathbb{E}\left[q_{k,l,h} \cdot \widetilde{X}_{k,l,h}^2\right] + \frac{m\log d}{\eta} + \gamma H^{k} \label{ineq: apply_omd} \\
&\leq m \eta \sum_{h=1}^H \frac{H^{2k-2} d^4 C^2}{\gamma} + \frac{m\log d}{\eta} + \gamma H^{k} \label{ineq: bcp_apply_bounded_var} \\
&= \frac{\eta H^{2k-1} m d^4 C^2}{\gamma} + \frac{m\log d}{\eta} + \gamma H^{k}  \label{ineq: bcp_apply_bounded_var_final} 
\end{align}

where: in \ref{eq: apply_prop_opt} we used Proposition \ref{prop: opt}, and also the Carathéodory decomposition of $\widetilde{p}_{k,l,h}$, in \ref{ineq: gamma_bound} we used the Holder's inequality utilizing the fact that $\|X_{k,l,h}\|_{\infty} \leq H^{k-1}$, in \ref{eq: unbiased_to_regret} we used Lemma \ref{bcp: properties} (1), in \ref{ineq: apply_omd} we used Lemma \ref{lem: omd_combcp}, and in \ref{ineq: bcp_apply_bounded_var} we used Lemma \ref{bcp: properties} (3).

\medskip

Now, recall that the maximum reward per time step (i.e., a meta-day) is $H^{k-1}$, while the horizon of our online learning setting for each learner is $H$. Therefore, we need each learner to achieve external regret sublinear to $H^{k}$, in order our algorithm to achieve no-swap-regret learning. 

Setting $\eta' = H^{k-1} \eta$, we have:

\begin{align}
\reg 
& \leq  \frac{\eta' H^{k} m d^4 C^2}{\gamma} + \frac{H^{k-1}m\log d}{\eta'} + \gamma H^{k}  \\
& = H^{k-1} \left( \frac{\eta' H m d^4 C^2}{\gamma} + \frac{m\log d}{\eta'} + \gamma H  \right)
\end{align}

Now, the following must hold (required for the Online Mirror Descent Lemma):

$$\eta \leq \frac{\gamma}{H^{k-1}C^2d^3\sqrt{m}}\Rightarrow \eta' \leq \frac{\gamma}{C^2d^3\sqrt{m}}.$$

\medskip






Setting $\eta'=\frac{1}{d^3 \sqrt{m} H^{2/3}}$, $\gamma= H^{-1/3} $, and $C=2$ we have: 

$$\reg\leq  3 H^{k-1} H^{2/3} d^3 m^{3/2} \log d. $$

\end{proof}

\medskip

\begin{theorem}[Theorem \ref{thm: swap_lazy_comband_combcp} restated]
    Setting $H=27\lfloor \log(T) \rfloor^3 d^9m^{9/2}\log^3 d$ and $K=\lfloor \log_H(T) \rfloor$, \textsc{Swap-ComBCP} satisfies the following swap-regret guarantee:
    $$ \swap \leq \frac{45\log (d\log T)T}{ \log T}.$$ 
\end{theorem}

\begin{proof}
    
Setting $g(m,d) = 3 d^3 m^{3/2} \log d$ and putting everything together in \ref{basic_step}, we have:
\begin{align}
    \swap_{T}({\phi}) 
    & \leq \frac{1}{K}\sum_{k=1}^{K-1}\sum_{l=1}^{\ceil*{\frac{T}{H^{k}}}}\reg(\textsc{Lazy-ComBCP}_{k,l})+\frac{T}{K} \\
    & \leq \frac{1}{K}\sum_{k=1}^{K-1}\sum_{l=1}^{\ceil*{\frac{T}{H^{k}}}}g(m,d) H^{k-1}H^{2/3} +\frac{T}{K} \\
    & \leq \frac{1}{K}\sum_{k=1}^{K-1}\left(\frac{T}{H^{k}}+1\right)g(m,d) H^{k-1}H^{2/3} +\frac{T}{K} \\
    & = \frac{T g(m,d)}{H^{1/3}} +\frac{1}{K}H^{2/3}g(m,d)\sum_{j=0}^{K-2} H^{j}+\frac{T}{K} \\
     & = \frac{T g(m,d)}{H^{1/3}} +\frac{1}{K}H^{2/3}g(m,d) \frac{H^{K-1}-1}{H-1}+\frac{T}{K} \\
     & \leq \frac{T g(m,d)}{H^{1/3}} +\frac{2}{K}g(m,d) \frac{H^{K-1}}{H^{1/3}}+\frac{T}{K} \\
\end{align}

We set $H=\lfloor \log(T) \rfloor^3(g(m,d))^3$ and $K=\lfloor \log_H(T) \rfloor$ and we obtain:

\begin{align}
   \swap_{T}({\phi}) 
    & \leq  \frac{T}{\lfloor \log(T) \rfloor} +2\frac{T}{KH\lfloor \log(T) \rfloor}+\frac{T}{K} \\
    & \leq  2\frac{T}{\lfloor \log(T) \rfloor}+\frac{T}{\lfloor \log_H(T) \rfloor} \\
    & \leq  2\frac{T}{ \log(T)-1}+\frac{T}{\log_H(T)-1} \\
    & \leq  4\frac{T}{ \log(T)}+2\frac{T}{\log_H(T)} \\
    & =  4\frac{T}{ \log(T)}+2\frac{T\log(H)}{\log(T)} \\
    & \leq  \frac{\left(4+6\log(\log(T))+6\log(g(m,d))\right)T}{ \log(T)} \\
    & \leq  \frac{45\log (d\log T)T}{ \log T}
\end{align}

\end{proof}

\newpage

\subsection{Issue with the exploration scheme of \textsc{CombExp}}\label{appendix: combexp}

In settings where the components of the action set appear with highly uneven frequencies—some occurring in far more actions than others—\textsc{CombExp} faces a challenge in its exploration strategy. Infrequently occurring components might be insufficiently explored, which becomes problematic if these components yield high rewards. As a result, the learner risks overlooking them, potentially leading to significantly suboptimal performance and increased regret. Next we construct one such example in the setting of online path planning to demonstrate the issue.

The idea is to construct a DAG where all edges but one participate in an exponential number of paths and one edge in a small number of paths. The latter edge gives high reward while all the other edges give small reward. For simplicity consider that the rewards are time invariant, so the no-regret learning objective reduces to correctly identifying the high-reward edge and selecting paths that include it. We show that, in such a construction, \textsc{CombExp} requires, with high probability, an exponential number of rounds to identify the high-reward edge. Consequently, for any horizon that is only polynomially large, the algorithm incurs regret that grows linearly with the horizon. Next we provide the formal construction of the counterexample.

Fix an integer $n\ge1$. Define the directed acyclic graph $G=(V,E)$ where the set of vertices is
\[
V = \{S,D\}\,\cup\,\{A_i,B_i\mid i=1,2,\dots,n\}.
\]
and the set of edges is defined as:
  \[
  \begin{aligned}
  E \;=&\ \{(S,D),(S,A_1),(S,B_1)\}\\[0.4em]
   &\ \cup\ \{(A_i,A_{i+1}),(A_i,B_{i+1}),(B_i,A_{i+1}),(B_i,B_{i+1}) \mid i=1,\dots,n-1\}\\[0.4em]
   &\ \cup\ \{(A_n,t),(B_n,t)\}.
  \end{aligned}
  \]
In this DAG the number of nodes is $|V|=2n+2$, the number of edges $|E|=4n+1$ and the maximal path length 
is $n+1$. In particular, there are $2^n$ paths of length $n+1$ and one path of length $1$. Thus, in the binary representation of paths in this setting we have $d=4n+1$, $m=n+1$ and $|\mathcal{A}|=2^n+1$.

Regarding the usage of edges in the DAG we have the following:
\begin{itemize}
\item $(S,D)$ is used by exactly $1$ path.
\item Edges $(S,A_1)$, $(S,B_1)$, $(A_n,D)$, $(B_n,D)$ are each used by $2^{\,n-1}$ paths.
\item Internal edges $(X_i\to Y_{i+1})$ for $X,Y \in \{A,B\},\ 1\le i\le n-1$ are each used by $2^{\,n-2}$ paths.
\end{itemize}

We consider an online learning setting where actions are paths from the starting node $S$ to the destination node $D$ in the graph $G$ we described above.
In the edge incidence vectors $M\in \cal A$ of $S-D$ paths in $G$ suppose that the first coordinate $i=1$ corresponds to the (shortcut) edge $(S,D)$. That is, $M(1)=1$ if the path represented by binary vector $M$ contains the edge $(S,D)$, otherwise $M(1)=0$. 

Consider the reward sequence $R_t(i)=\mathbbm{1}[i=1],\ \forall t\in[T]$. We will study how \textsc{CombEXP} performs in this sequence. Obviously, the best fixed action $M^*$ in hindsight is selecting the path that uses the shortcut edge $(S,D)$. $M^*$ gives reward $1$ at each timestep, while all other actions give zero reward. 
The initial distribution assigns exponentially small probability to the optimal edge :
$\mu_{\min}=m\mu^0(1)=\frac{1}{2^n+1}$. Moreover, we observe that in the decomposition $p_0$ of $q'_0=\mu^0$ the optimal action $M^*$ appears with weight $\mu_{\min}$. This means that at $t=1$ \textsc{CombEXP} plays $M^*$ with probability $\frac{1}{2^n+1}$. With probability $1-\frac{1}{2^n+1}$ an action $M \neq M^*$ is played and reward $r_1=0$ is received. In this case, $\tilde{R}_t$ is the zero vector and no update is performed on $q$, that is $q_1=q_0=\mu^0$. This pattern is repeated in the next timesteps (i.e. $q_t=q_{t-1}, r_t=0$) and with probability at least $1-\frac{T}{2^n+1}$ \textsc{CombEXP} has realized regret $\sum_{t=1}^T R_t\cdot M^* - r_t =T$.

Apart from paths other combinatorial structures can create this issue as well. For example, consider partitions of $n$ indistinguishable elements to $k$ labeled groups under the binary representation of \cite{kontogiannis2025efficient} and suppose that the learner receives non zero reward only if it assigns all $n$ elements to one particular group. Then, the realized regret of \textsc{CombEXP} is $\Theta(T)$ with probability at least $1-\frac{T}{\binom{n+k-1}{k-1}}$.

\begin{algorithm}[tb]
\small
   \caption{\textsc{CombEXP} \cite{combes2015combinatorial}}
   \label{alg:CombEXP}
\begin{algorithmic}
   \STATE {\bf Initialization:} Let 
   $\mu^0(i)=\frac{1}{m|\cal A|}\sum_{M\in \cal A}M(i),\ \forall i\in[d]$,\quad $\mu_{\min}=\min_{i\in[d]} m\mu^0(i)$\  and  $\underline{\lambda} = \lambda_{>0,\min}(\mathbb{E}_{M\sim \mathcal{U}(\cal A)}[MM^\top])$ \STATE \qquad\qquad\qquad\quad\ Set $q_0=\mu^0$, $\gamma=\frac{\sqrt{m\log \mu_{\min}^{-1}}} {\sqrt{m\log \mu_{\min}^{-1}}+\sqrt{C(Cm^2d+m)T}}$ and $\eta=\gamma C$, with $C=\frac{\underline{\lambda}}{m^{3/2}}$. 
   \FOR{$ t=1,...T$}
   \STATE \emph{Mixing:} Let $q'_{t-1}=(1-\gamma)q_{t-1}+\gamma\mu^0$. \vspace{.5mm}
   \STATE \emph{Decomposition:} Select a distribution $p_{t-1}$ over $\mathcal{A}$ such that $\sum_{M} p_{t-1}(M) M=mq'_{t-1}$. \vspace{.5mm}
   \STATE \emph{Sampling:} Select a random arm $M_t$ with distribution $p_{t-1}$ and incur a reward $r_t = R_t \cdot M_t$.  
   \STATE \emph{Estimation:} Let $\Sigma_{t-1}=\underset{M\sim p_{t-1}}{\mathbb{E}}\left[ MM^{\top}\right]$. Set $\tilde{R}_t = r_t\Sigma_{t-1}^{+}M_t$, where  $\Sigma_{t-1}^{+}$ is the pseudo-inverse of $\Sigma_{t-1}$. \vspace{.5mm}
   \STATE \emph{Update:} Set $\tilde{q}_t(i) \propto q_{t-1}(i) \exp(\eta \tilde{R}_{t}(i)),\; \forall i\in[d]$. \vspace{.5mm}
   \STATE \emph{Projection:} Set $q_t$ to be the projection of $\tilde{q}_t$ onto the set $\mathcal{P}$ using the KL divergence. 
   \ENDFOR
\end{algorithmic}
\normalsize
\end{algorithm}

\newpage

\section{\textsc{Swap-ComBand}} \label{appendix: comband}

\begin{algorithm}[h]
\caption{\textsc{Lazy-ComBand$_{k,l}$}}\label{alg:lazy_comband}
\begin{algorithmic}[1]
\STATE \textit{Initialize} $\widetilde{p}_{k,l,1} = [1/|\mathcal{A}|, ..., 1/|\mathcal{A}|] \in \Delta^{|\mathcal{A}|}$, $\mu \in \Delta^{|\mathcal{A}|}$, $\gamma=\frac{1}{H^{1/3}}$, $\eta=\frac{\lambda_{\text{min}}}{H^{k - 1/3} m}$
\STATE 
    $H' = \begin{cases}\qquad H &\text{ if }  l\leq\frac{T}{H^k} \\
    \ceil*{\frac{T-\floor*{\frac{T}{H^k}}H^k}{H^{k-1}}}  &\text{ othw. }
    \end{cases}$ {\color{blue} /$^*$Play for $H$ meta-days or until the time limit $T$ is reached$^*$/}\FOR{$h=1,2, \ldots, H' $}
\STATE \emph{Mixing:} Let $p_{k,l,h}=(1-\gamma)\widetilde{p}_{k,l,h}+\gamma\mu$ 
\STATE {\textit{Meta-day}:} $\text{Fix } p_{k,l,h} \text { for } H^{k-1} \text { days and aggregate the rewards of these days:}$  
\STATE $$\widetilde{X}_{k,l,h}(i)=\sum_{\tau=(\ell-1) H^k+(h-1) H^{k-1}+1}^{\min((\ell-1) H^k+h H^{k-1},T)} \widetilde{R}_{\tau}(i), \quad {\forall i \in[d]}$$
\STATE \emph{Update:} $\widetilde{p}_{k,l,h+1}(M) = \widetilde{p}_{k,l,h}(M) \exp(\eta \widetilde{X}_{k,l,h}),\; \forall M\in \mathcal{A}$
\ENDFOR
\end{algorithmic}
\end{algorithm}

\begin{lemma}(Corollary 12, in \cite{cesa2012combinatorial})\label{lem:projection}
    For all $t$ and $x \in \mathbb{R}^d$, it holds that $\Sigma_t \Sigma_t^{+} x = {x}^{\parallel}$, where ${x}^{\parallel}$ is the orthogonal projection of $x$ on the linear span of $\mathcal{A}$.
\end{lemma}

\begin{lemma}\label{comband: properties}
Let $\lambda_{\min}$ be the smallest nonzero eigenvalue of the co-occurrence matrix $\Sigma$ under the law of $\mu$. For all $k \in \{2,...,K+1\}$, $l \in \left[\ceil*{\frac{T}{H^{k}}}\right]$ and $h \in [H]$, $\textsc{Lazy-ComBand}$ satisfies the following properties:

\begin{enumerate}
    \item (Unbiasedness): $ \mathbb{E}\left[\sum_{M \in \mathcal{A}} \widetilde{p}_{k,l,h}(M) {\widetilde{X}}_{k,l,h}\cdot M\right] = \mathbb{E}\left[\sum_{M \in \mathcal{A}} \widetilde{p}_{k,l,h}(M) X_{k,l,h}\cdot M\right]$
    \item (Boundedness): $ \left|\widetilde{X}_{k,l,h} \cdot M\right| \leq \frac{H^{k-1} m}{\gamma \lambda_{min}} $
    \item (Bounded variance): $ \mathbb{E}\left[\sum_{M \in \mathcal{A}} \widetilde{p}_{k,l,h}(M) \widetilde{X}_{k,l,h}^2 \cdot M\right] \leq \frac{H^{2k-2} m^2 d}{\gamma \lambda_{\text{min}}} $
\end{enumerate}

\end{lemma}

\begin{proof}

$\bullet \quad \quad$ 1.

We have that,

\begin{align}
    \sum_{\tau=1}^{H^{k-1}} \mathbb{E}_{k,l,h,\tau}\left[\widetilde{R}_{k,l,h,\tau}\right]
    &= \sum_{\tau=1}^{H^{k-1}} \mathbb{E}_{k,l,h,\tau}\left[r_{k,l,h,\tau}\Sigma_{k,l,h,\tau}^{+}M_{k,l,h,\tau}\right] \\
    &= \sum_{\tau=1}^{H^{k-1}} \mathbb{E}_{k,l,h,\tau}\left[\Sigma_{k,l,h,\tau}^{+}M_{k,l,h,\tau} M_{k,l,h,\tau} \cdot {R}_{k,l,h,\tau}\right] \\
    &= \sum_{\tau=1}^{H^{k-1}} \left( \Sigma_{k,l,h,\tau}^{+} \Sigma_{k,l,h,\tau} R_{k,l,h,\tau}\right) \\
    &= \sum_{\tau=1}^{H^{k-1}} {R_{k,l,h,\tau}^{\parallel}} \label{proj_step}
\end{align}

where in \ref{proj_step} we used Lemma \ref{lem:projection}. Now, based on the above, for any $M \in \mathcal{A}$, we have that,

\begin{align}
    \sum_{\tau=1}^{H^{k-1}} \mathbb{E}_{k,l,h,\tau}\left[\widetilde{R}_{k,l,h,\tau}\right] \cdot M
    &= \sum_{\tau=1}^{H^{k-1}} {R_{k,l,h,\tau}^{\parallel}} \cdot M \\
    &= \sum_{\tau=1}^{H^{k-1}} R_{k,l,h,\tau} \cdot M \label{proj_to_normal_step} \\
    &= X_{k,l,h,\tau} \cdot M \label{prefinal_exp_comband}
\end{align}

where in \ref{proj_to_normal_step} we used the fact that $R_{k,l,h,\tau} = {R_{k,l,h,\tau}^{\parallel}} + R^{\perp}_{k,l,h,\tau}$, and also that $R^{\perp}_{k,l,h,\tau} \cdot M = 0$. The above also implies that

\begin{align}
    \mathbb{E}_{k,l,h,\tau}\left[\sum_{M \in \mathcal{A}} \widetilde{p}_{k,l,h}(M) \widetilde{X}_{k,l,h}\cdot M \right] 
    &= \mathbb{E}_{k,l,h,\tau}\left[\sum_{M \in \mathcal{A}} \widetilde{p}_{k,l,h}(M) \sum_{\tau=1}^{H^{k-1}} \widetilde{R}_{k,l,h,\tau}\cdot M \right] \\
    &=  \sum_{M \in \mathcal{A}} \widetilde{p}_{k,l,h}(M) \mathbb{E}_{k,l,h,\tau}\left[\sum_{\tau=1}^{H^{k-1}} \widetilde{R}_{k,l,h,\tau} \cdot M\right] \\
    &=  \sum_{M \in \mathcal{A}} \widetilde{p}_{k,l,h}(M) \mathbb{E}_{k,l,h,\tau} \widetilde{X}_{k,l,h} \cdot M \label{final_step_exp_comband}
\end{align}

where in \ref{final_step_exp_comband} we used \ref{prefinal_exp_comband}. Finally, taking expectations in both sides in \ref{final_step_exp_comband}, and using the tower property, we get the desired result.

$\bullet \quad \quad$ 2.

We have that,

\begin{align}
    \left|\widetilde{X}_{k,l,h} \cdot M\right| &\leq \sum_{\tau=1}^{H^{k-1}} \left|r_{k,l,h,\tau}\right| \left|M^T \Sigma_{k,l,h,\tau}^{+}M_{k,l,h,\tau}\right|  \\
    &\leq H^{k-1}  \left\| \Sigma_{k,l,h,\tau}^{+} \right \|_2 \max_{M \in \mathcal{A}} \|M\|^2 \label{comband: bound1} \\
    &\leq \frac{H^{k-1} m}{\lambda_{\text{min}}(\Sigma_{k,l,h,\tau})} \\
    &\leq \frac{H^{k-1} m}{\gamma\lambda_{\text{min}}} \label{comband: bound2}
\end{align}

where in \ref{comband: bound1} we used the fact that $|r_t| \leq 1$, and in \ref{comband: bound2} we used the Weyl's inequality on $\lambda_{\min }\left(\Sigma_{k,l,h,\tau}\right) \geq \gamma \lambda_{\min }$, where we denote by $\lambda_{\min }\left(\Sigma_{k,l,h,\tau}\right)$ the smallest nonzero eigenvalue of the co-occurrence matrix under the law of $\widehat{p}_{k,l,h,\tau}$, and by $\lambda_{\min}$ the smallest nonzero eigenvalue of the co-occurrence matrix under the law of $\mu$.

\bigskip

$\bullet \quad \quad$ 3. 

\medskip

For all $k \in \{1,...,K-1\}$, $l \in \left[\ceil*{\frac{T}{H^{k}}}\right]$ and $h \in [H]$, we have:

Let $Var = \mathbb{E}_{k,l,h}\left[\sum_{M \in \mathcal{A}} \widetilde{p}_{k,l,h}(M) \left(\widetilde{X}_{k,l,h}\cdot M\right)^2\right]$

\begin{align}
     Var
    & \leq m \mathbb{E}_{k,l,h}\left[\sum_{M \in \mathcal{A}} \widetilde{p}_{k,l,h}(M) \widetilde{X}_{k,l,h}^2\cdot M\right] \label{comband: var__1} \\
    & = m \mathbb{E}_{k,l,h}\left[\widehat{M}_{k,l,h} \cdot \widetilde{X}_{k,l,h}^2\right] \label{comband: var__2} \\
    & = m \mathbb{E}_{k,l,h}\left[\widehat{M}_{k,l,h} \cdot \left(\sum_{\tau=1}^{H^{k-1}} \widetilde{R}_{k,l,h,\tau}\right)^2\right] \\
     & \leq H^{k-1} m \mathbb{E}_{k,l,h}\left[\widehat{M}_{k,l,h} \cdot \sum_{\tau=1}^{H^{k-1}} \widetilde{R}_{k,l,h,\tau}^2\right] \label{comband: var__3} \\
    & = H^{k-1} m  \sum_{\tau=1}^{H^{k-1}} \mathbb{E}_{k,l,h}\left[\widehat{M}^2_{k,l,h} \cdot  \widetilde{R}_{k,l,h,\tau}^2 \right]\label{comband: var__5}  \\
    & \leq H^{k-1} m  \sum_{\tau=1}^{H^{k-1}} \mathbb{E}_{k,l,h}\left[M_{k,l,h,\tau}^T \Sigma_{k,l,h,\tau}^{+} \widehat{M}_{k,l,h} \widehat{M}_{k,l,h}^T  \Sigma_{k,l,h,\tau}^{+} M_{k,l,h,\tau} \right] \label{comband: var__55} \\
    & \leq H^{k-1} m \sum_{\tau=1}^{H^{k-1}} \mathbb{E}_{k,l,h}\left[\lambda_{\text{max}}(\widehat{M}_{k,l,h} \widehat{M}_{k,l,h}^T) M_{k,l,h,\tau}^T \Sigma_{k,l,h,\tau}^{+2} M_{k,l,h,\tau} \right] \label{comband: var___6} \\
    & \leq H^{k-1} m \cdot \max_{M \in \mathcal{A}} \lambda_{\text{max}}(MM^T) \cdot \sum_{\tau=1}^{H^{k-1}} \mathbb{E}_{k,l,h}\left[ M_{k,l,h,\tau}^T \Sigma_{k,l,h,\tau}^{+2} M_{k,l,h,\tau} \right] \\
    & \leq H^{k-1} m^2 \sum_{\tau=1}^{H^{k-1}} \mathbb{E}_{k,l,h}\left[ M_{k,l,h,\tau}^T \Sigma_{k,l,h,\tau}^{+2} M_{k,l,h,\tau} \right] \label{comband: var___7} 
\end{align}

\medskip

where in \ref{comband: var__1} we used the Holder's inequality $(a_1 + ... + a_N)^2 \leq N\sum_{i=1}^{N}a_i^2$, in \ref{comband: var__2} we defined $\widehat{M}_{k,l,h}$ to be a random pure action under the law of $\widetilde{p}_{k,l,h}$, independent of $M_{k,l,h,\tau}$ which is under the law of $\widehat{p}_{k,l,h}$, in \ref{comband: var__3} we used the inequality $(a_1 + ... + a_N)^2 \leq N\sum_{i=1}^{N}a_i^2$, in \ref{comband: var__5} we used the fact $\widehat{M}_{k,l,h} = \widehat{M}_{k,l,h}^2$ (squares are in an element-wise manner), in \ref{comband: var__55} we used the the fact that the bandit feedback reward $r_{k,l,h,\tau} \leq 1$, in \ref{comband: var___6} we used the inequality $x^TABAx^T \leq \lambda_{\text{max}}(B)x^TA^2x$ for any vector $x \in \mathbb{R}^d$ and symmetric matrices $A,B \in \mathbb{R}^{d\times d}$, in \ref{comband: var___7} we used the fact that $\max_{M \in \mathcal{A}} \lambda_{\text{max}}(MM^T) \leq m$.

\medskip

The next thing to do is to bound each summand of the above sum. Before doing so, we note that each of these summands is an expectation under the law of the master's distribution $\widehat{p}_{k,l,h,\tau}$.  Let $\lambda_i(\Sigma_{k,l,h,\tau})$ be the eigenvalue of $\Sigma_{k,l,h,\tau}$ corresponding to the eigenvector $u_i(\Sigma_{k,l,h,\tau})$. For brevity, we abuse notation and instead simply use $\lambda_{\tau,i}$ and $u_{\tau,i}$. Moreover, let $\lambda_{\tau,1}, ..., \lambda_{\tau,r}$ be the $r$ nonzero eigenvalues with corresponding eigenvectors $u_{\tau,1}, ..., u_{\tau,r}$. Based on the above, using Lemma \ref{lem:key}, we have the following:

\begin{align}
\mathbb{E}_{k,l,h}\left[ M_{k,l,h,\tau}^T \Sigma_{k,l,h,\tau}^{+2} M_{k,l,h,\tau} \right] 
& \leq d \lambda^{-1}_{\text{min}}(\Sigma_{k,l,h,\tau}) \label{comband: variance_9.5} \\
& \leq \frac{d}{\gamma \lambda_{\text{min}}} \label{comband: variance_11} 
\end{align}

where in \ref{comband: variance_11} we used the Weyl's inequality applying to $\lambda_{\text{min}}(\Sigma_{k,l,h,\tau}) \geq \gamma\lambda_{\text{min}}$.   

\medskip

Putting the above all together, we have the following:

\begin{align}
\mathbb{E}_{k,l,h}\left[\sum_{M \in \mathcal{A}} \widetilde{p}_{k,l,h}(M) \widetilde{X}_{k,l,h}^2 \cdot M\right] 
& \leq \frac{H^{2k-2} m^2 d}{\gamma \lambda_{\text{min}}}
\end{align}

Finally, taking expectations in both sides we get the desired result:

\begin{align}
\mathbb{E}\left[\sum_{M \in \mathcal{A}} \widetilde{p}_{k,l,h}(M) \widetilde{X}_{k,l,h}^2 \cdot M\right] 
& \leq \frac{H^{2k-2} m^2 d}{\gamma \lambda_{\text{min}}}
\end{align}

\end{proof}

\medskip




Requiring $\eta \leq \frac{\gamma \lambda_{\text{min}}}{H^{k-1}m}$ (this condition will be verified later), \textsc{Lazy-ComBand} satisfies the following Online Mirror Descent lemma:

\begin{lemma}[Online Mirror Descent on $\mathcal{A}$]\label{lem: omd_comband}
    If $\eta |\widetilde{X}_{k,l,h} \cdot M| \leq 1$, it holds that
    \begin{align}
    \mathbb{E}\left[\sum_{h=1}^H M^* \cdot \widetilde{X}_{k,l,h} -\sum_{h=1}^H \sum_{M \in \mathcal{A}} \widetilde{p}_{k,l,h}(M) \widetilde{X}_{k,l,h} \cdot M\right]  
    &\leq \eta \mathbb{E}\left[\sum_{h=1}^H \sum_{M \in \mathcal{A}} \widetilde{p}_{k,l,h}(M) \left(\widetilde{X}_{k,l,h}\cdot M\right)^2\right]  + \frac{\log |\mathcal{A}|}{\eta}. \nonumber 
    \end{align} 
\end{lemma}

\medskip

\begin{theorem}[No-External-Regret]\label{thm:comband_no_regret_}
The external regret of $\textsc{Lazy-ComBand}_{k,l}$ is at most $H^{k-1} H^{2/3} \left(m d + \frac{m^2 \log d}{\lambda_{\text{min}}} + {2} \right)$.
\end{theorem}

\begin{proof}

Aiming to bound the external regret of \textsc{Lazy-ComBand}, we get, 

\begin{align}
\reg
&= \max_{M \in \mathcal{A}}\left\{\sum_{h=1}^{H}X_{k,l,h}\cdot M \right\} - \sum_{h=1}^{H}\mathbb{E}\left[\sum_{M \in \mathcal{A}}p_{k,l,h}(M) X_{k,l,h}\cdot M \right] \\
&= \sum_{h=1}^{H}X_{k,l,h}\cdot M_{k,l}^*  - \sum_{h=1}^{H}\mathbb{E}\left[\sum_{M \in \mathcal{A}}p_{k,l,h}(M) X_{k,l,h}\cdot M \right] \\
&= 
\mathbb{E}\left[\sum_{h=1}^{H}X_{k,l,h}\cdot M_{k,l}^*\right]  - \sum_{h=1}^{H}\mathbb{E}\left[\sum_{M \in \mathcal{A}}p_{k,l,h}(M) X_{k,l,h}\cdot M \right] \\
&= \mathbb{E}\left[\sum_{h=1}^{H}{X}_{k,l,h}\cdot M_{k,l}^*\right]  -\sum_{h=1}^{H}\mathbb{E}\left[\sum_{M \in \mathcal{A}}[(1-\gamma) \widetilde{p}_{k,l,h}(M) + \gamma \mu(M)] {X}_{k,l,h}\cdot M \right] \\
&= \mathbb{E}\left[\sum_{h=1}^{H}{X}_{k,l,h}\cdot M_{k,l}^* - \sum_{h=1}^{H}\sum_{M \in \mathcal{A}} \widetilde{p}_{k,l,h}(M) {X}_{k,l,h}\cdot M\right] + \gamma \sum_{h=1}^{H}\mathbb{E}\left[\sum_{M \in \mathcal{A}} \widetilde{p}_{k,l,h}(M) {X}_{k,l,h}\cdot M \right] \nonumber \\ & - \gamma \sum_{h=1}^{H}\mathbb{E}\left[\sum_{M \in \mathcal{A}} \mu(M) {X}_{k,l,h}\cdot M \right] \label{comband: step1} \\
&\leq \mathbb{E}\left[\sum_{h=1}^{H}{X}_{k,l,h}\cdot M_{k,l}^* - \sum_{h=1}^{H}\sum_{M \in \mathcal{A}} \widetilde{p}_{k,l,h}(M) {X}_{k,l,h}\cdot M\right] + 2\gamma  H^{k} \label{comband: step2} \\
&= \mathbb{E}\left[\sum_{h=1}^{H}{\widetilde{X}}_{k,l,h}\cdot M_{k,l}^* - \sum_{h=1}^{H}\sum_{M \in \mathcal{A}} \widetilde{p}_{k,l,h}(M) {\widetilde{X}}_{k,l,h}\cdot M\right] + 2\gamma  H^{k} \label{comband: step3} \\
& \leq \eta \mathbb{E}\left[ \sum_{h=1}^H \sum_{M \in \mathcal{A}} \widetilde{p}_{k,l,h}(M) \left(\widetilde{X}_{k,l,h}\cdot M\right)^2 \cdot M \right] + \frac{\log |\mathcal{A}|}{\eta} + 2\gamma  H^{k} \label{comband: step4} \\
& \leq \frac{\eta H^{2k-1} m^2 d}{\gamma \lambda_{\text{min}}} + \frac{m\log d}{\eta} + 2\gamma  H^{k} \label{comband: step5} 
\end{align}

\medskip

where in \ref{comband: step2} we used the fact that $|{R}_{k,l,h} \cdot M| \leq 1$, in \ref{comband: step3} we used the unbiasedness of the estimator, in \ref{comband: step4} we applied  Lemma \ref{lem: omd_comband}, in \ref{comband: step5} we used Lemma \ref{comband: properties} (2) and that $|\mathcal{A}| \leq d^m$.

\medskip

Setting $\eta' = H^{k-1} \eta$, we have:

\begin{align}
    \reg_{k} 
    & \leq H^{k-1} \left( \frac{\eta' H m^2 d}{\gamma \lambda_{\text{min}}} + \frac{m\log d}{\eta'} + 2\gamma  H \right) 
\end{align}

Now, setting $\gamma = \frac{1}{H^{1/3}}$ and $\eta'=\frac{\gamma \lambda_{\text{min}}}{H^{1/3} m} = \frac{\lambda_{\text{min}}}{H^{2/3} m} \Rightarrow \eta = \frac{\lambda_{\text{min}}}{H^{k - 1/3} m}$, we get,

\begin{align}
    \reg_{k} 
    & \leq H^{k-1} \left( H^{2/3} m d + \frac{H^{2/3} m^2 \log d}{\lambda_{\text{min}}} + {2  H^{2/3} } \right) \\ 
    & = H^{k-1} H^{2/3} \left(m d + \frac{m^2 \log d}{\lambda_{\text{min}}} + {2} \right)
\end{align}

\medskip

Notice that with the above configuration the condition $\eta \leq \frac{\gamma \lambda_{\text{min}}}{H^{k-1}m}$ holds.

\end{proof}

\medskip

\begin{theorem}[No-Swap-Regret]\label{thm: swap_lazy_comband}
    If $H=\lfloor \log(T) \rfloor^3(g(m,d,\lambda_{\text{min}}))^3$ and $K=\lfloor \log_H(T) \rfloor$, \textsc{SwapComBand} achieves swap-regret 
    $$ \swap_{T}({\phi}) \leq \mathcal{O}\left(\frac{T\log(d\log T)}{ \log(T)}\right) ,$$
    where $g(m,d,\lambda_{\text{min}}) = \left(m^2 d + \frac{m^2 \log d}{\lambda_{\text{min}}} + {2 m} \right)$.
\end{theorem}

\begin{proof}

Let $g = \left(m d + \frac{m^2 \ceil*{\log d}}{\lambda_{\text{min}}} + {2} \right)$. Putting the no-regret guarantees in \ref{basic_step}, the following hold:

\begin{align}
    \swap_{T}({\phi}) 
    & \leq \frac{1}{K}\sum_{k=1}^{K-1}\sum_{l=1}^{\ceil*{\frac{T}{H^{k}}}}\reg(\textsc{Lazy-CombAlg}_{k,l})+\frac{T}{K} \\
    & \leq \frac{1}{K}\sum_{k=1}^{K-1}\sum_{l=1}^{\ceil*{\frac{T}{H^{k}}}}g(m,d,\lambda_{\text{min}}) H^{k-1}H^{2/3} +\frac{T}{K} \\
    & \leq \frac{1}{K}\sum_{k=1}^{K-1}\left(\frac{T}{H^{k}}+1\right)g(m,d,\lambda_{\text{min}}) H^{k-1}H^{2/3} +\frac{T}{K} \\
    & = \frac{T g(m,d,\lambda_{\text{min}})}{H^{1/3}} +\frac{1}{K}H^{2/3}g(m,d,\lambda_{\text{min}})\sum_{j=0}^{K-2} H^{j}+\frac{T}{K} \\
     & = \frac{T g(m,d,\lambda_{\text{min}})}{H^{1/3}} +\frac{1}{K}H^{2/3}g(m,d,\lambda_{\text{min}}) \frac{H^{K-1}-1}{H-1}+\frac{T}{K} \\
     & \leq \frac{T g(m,d,\lambda_{\text{min}})}{H^{1/3}} +\frac{2}{K}g(m,d,\lambda_{\text{min}}) \frac{H^{K-1}}{H^{1/3}}+\frac{T}{K} 
\end{align}

We set $H=\lfloor \log(T) \rfloor^3(g(m,d,\lambda_{\text{min}}))^3$ and $K=\lfloor \log_H(T) \rfloor$ and we obtain:

\begin{align}
   \swap_{T}({\phi}) 
    & \leq  \frac{T}{\lfloor \log(T) \rfloor} +2\frac{T}{KH\lfloor \log(T) \rfloor}+\frac{T}{K} \\
    & \leq  2\frac{T}{\lfloor \log(T) \rfloor}+\frac{T}{\lfloor \log_H(T) \rfloor} \\
    & \leq  2\frac{T}{ \log(T)-1}+\frac{T}{\log_H(T)-1} \\
    & \leq  4\frac{T}{ \log(T)}+2\frac{T}{\log_H(T)} \\
    & =  4\frac{T}{ \log(T)}+2\frac{T\log(H)}{\log(T)} \\
    & \leq  \frac{\left(4+6\log(\log(T))+6\log(g(m,d,\lambda_{\text{min}}))\right)T}{ \log(T)}
\end{align}

\end{proof}

\medskip

\begin{remark}
    In \textsc{Swap-ComBand}, under the minimal assumption of an efficient linear minimization oracle for $\mathcal{A}$, $\mu$ can be efficiently constructed as the uniform distribution over a barycentric spanner (see Proposition \ref{prop:barycentric_algo}) and achieve $\lambda_{\text{min}}(\mu) \geq \frac{1}{4 d^3}$ (see Lemma \ref{lem: 2016_lem}). 
    As an alternative, $\mu$ could be set to the uniform distribution over $\mathcal{A}$, which in many cases has $1/\lambda_{\text{min}}(\mu)=\mathcal{O}(\text{poly}(d))$ (see \cite{cesa2012combinatorial}).
\end{remark}


\section{Per-Iteration Complexity in Well-Studied Combinatorial Settings}\label{appendix: apps}

\subsection{Applications of \textsc{Swap-ComBCP}} \label{appendix: apps_combcp}
We present some important combinatorial bandit settings where \textsc{Swap-ComBCP} is applicable (i.e. the $L1$ norm of the action vectors is equal to a fixed value $m$) and efficiently implementable (i.e. there exist efficient algorithms for decomposition, projection and approximate barycentric spanner calculation). For all settings we consider here, except for paths, efficient decomposition and projection algorithms are proposed in \cite{suehiro2012online}. For paths in DAGs the path polytope has a succinct description which allows efficient decomposition and projection with interior point methods (see \cite{boyd2004convex}, page 545]). Efficient calculation of an approximate barycentric spanner can be achieved if one has access to an efficient linear minimization oracle (LMO) over the action set $\mathcal{A}$ (see Proposition \ref{prop:barycentric_algo}), thus for the settings of interest it remains to show that such an LMO exists.
\begin{itemize}
 \item \textbf{m-sets:} The most basic instance of combinatorial bandits are m-sets, where the decision maker at each timestep selects a subset of $m$ out of $d$ items. In binary representation the action space is defined as $\mathcal{A} = \{M \in \{0,1\}^d \mid \sum_{i} M_i = m\}$. In this setting it trivially holds that $|M|_1=m$ and the LMO greedily chooses the top m arms with the lowest cost.
   
 \item \textbf{Spanning Trees and Graph Matroid Bandits (k-forests):} We examine an online decision-making problem where, at each timestep, the decision maker selects a spanning tree from an undirected graph $G=(V,E)$. This type of problem is relevant in certain mobile communication networks, where a minimum-cost subnetwork must be chosen in each timestep to ensure the network remains connected. Formally, the action space is defined as follows: $\mathcal{A} = \{M \in [0,1]^{|E|} \mid \text{the set of edges}\{e \mid M_e = 1\} \text{ forms a spanning tree of } G\}$. We observe that for all $M\in\mathcal{A}$ it holds that $|M|_1=|V|-1$. As for the LMO, it can be efficiently implemented by minimum spanning tree algorithms, such as Kruskal’s algorithm.

Generalizing the previous setting, let $\mathcal{A}$ denote the set of $k$-forests in a graph $G = (V, E)$, that is $\mathcal{A} = \{M \in [0,1]^{|E|} \mid \text{the set of edges}\{e \mid M_e = 1\} \text{ does not contain a cycle in } G\text{ and }|M|_1=k\}$. It is known that $\mathcal{A}$ is  a bases family of a truncation of a graphic matroid. Online learning in such action spaces has been studied in previous works (e.g. in \cite{kveton2014matroid}). For the LMO, Kruskal’s algorithm with slight adaptations works.

\item \textbf{Permutations and Truncated Permutations:} 
Consider the complete bipartite graph $K_{m,m}$, and let $\mathcal{A}$ contain all perfect matchings, that is $\mathcal{A} = \{M \in [0,1]^{m \times m} \mid \sum_{j=1}^m M_{i,j} = 1$\text{ for all }$i = 1, \dots, m$\text{ and }$\sum_{i=1}^m M_{i,j} = 1$\text{ for all }$j = 1, \dots, m\}$. Note that $\mathcal{A}$ is also equivalent to the set of all permutations of $m$ items.  Online optimization over permutations models problems such as ranking of search results or matching workers to tasks. The setting can be generalized to selecting truncated permutations of $k$ out of $m$ elements. Equivalently, a truncated permutation is a maximal matching in the complete bipartite graph between $[k]$ and $[m]$. 
Search query results typically take the form of a truncated permutation, where only the top $k$ results, out of $m$ available ones, are displayed in decreasing order of relevance. In binary representation, truncated permutations of $k$ out of $m$ items satisfy the condition $|M|_1=k$ and for the LMO efficient algorithms for Maximum Bipartite Matching can be used. One simple option is the Ford-Fulkerson algorithm. 

\item \textbf{Paths:} Perhaps the most important setting in combinatorial bandits are path planning problems. We consider a directed acyclic graph $G=(V,E)$ with a source node $s$ and a sink $t$ where at each timestep the decision maker selects a path from $s$ to $t$. In binary representation, each action vector $M$ is the incidence vector of a path. To fit the requirements of \textsc{ComBCP}, all action vectors $M$ should have the same $L1$ norm, that is all paths should have the same length. Of course this property does not hold in an arbitrary DAG, however, as shown in \cite{gyorgy2007line}, it is possible to transform any DAG by adding at most $(K-2)(|V|-2)+1$ vertices and edges (with constant weight zero) and ensure that in the new DAG all paths from $s$ to $t$ have a length of $K$, where $K$ is the maximal path length in the original graph. This way, learning paths in the transformed DAG is equivalent to learning paths in the original DAG, and the condition $|M|_1=K$ is met. The LMO can be implemented with dynamic programming with running time linear in $|V|+|E|$.

\end{itemize}

\subsection{Applications of \textsc{Swap-ComBand}}\label{appendix: apps_comband}

Next we present some applications where \textsc{Swap-ComBand} is efficiently implementable. The implementation of \textsc{Swap-ComBand} requires efficient sampling from the MWU distribution, the exploration policy $\mu$ and the uniform distribution over $\mathcal{A}$ and efficient calculation of the co-occurrence matrix. In many settings these routines can be efficiently implemented and outperform those used in \textsc{Swap-ComBCP}. Moreover, in contrast to \textsc{Swap-ComBCP}, \textsc{Swap-ComBand} does not require that all action vectors $M$ have the same $L1$ norm, which in some applications can be restrictive.

\paragraph{Paths:} The main application we consider is online path planning in DAGs. As we saw in the previous subsection \textsc{Swap-ComBCP} can be applied in this setting (after transforming the DAG to meet the requirement that all s-t paths have the same length). However, in each iteration \textsc{Swap-ComBCP} will need to perform the costly operation of projection on the path polytope. A standard approach for this step would be using interior point methods (see \cite{boyd2004convex}, page 545]), which would take $(|E|+|V|)^3$ time. In \textsc{Swap-ComBand}, the co-occurrence matrix calculation takes $\mathcal{O}(|E|^2$) time using the techniques of \cite{takimoto2003path} and sampling from MWU takes $\mathcal{O}(|E|$) time using weight pushing \cite{takimoto2003path}. Sampling from the uniform distribution over $\mathcal{A}$ takes $\mathcal{O}(|E|+|V|$) time. For the exploration policy $\mu$ we can use a barycentric spanner, similarly to \textsc{Swap-ComBCP}. Thus, in path planning problems over DAGS \textsc{Swap-ComBand} can be efficiently implemented and its per iteration complexity is improved compared to \textsc{Swap-ComBCP}.

Path planning problems over DAGs can also model other classic combinatorial bandit settings, such as \textbf{m-sets} (see \cite{cesa2012combinatorial}).

\end{document}